%% file: main_arxiv_v2.tex
\documentclass[11pt]{amsart}
\usepackage[T1]{fontenc}
\usepackage[utf8]{inputenc}
\usepackage{fullpage}
\usepackage{amsmath,amssymb,amsthm,amsfonts,mathtools,bm}
\usepackage{graphicx}
\usepackage{tikz}
\usepackage{pgfplots}
\pgfplotsset{compat=1.18}
\usepackage{booktabs}
\usepackage{enumitem}
\setlist{leftmargin=*}
\usepackage{microtype}
\usepackage{url}
\usepackage[numbers,sort&compress]{natbib}
\usepackage[hidelinks]{hyperref}
\usepackage[capitalize,nameinlink]{cleveref}
\hypersetup{
  pdftitle={Axiomatic Positivstellensaetze for Function Algebras over Ordered Fields},
  pdfauthor={Nayoon Kim, Allen Gehret, Shenyuan Ma, Jakub Mareček},
  pdfsubject={Axiomatic strict and weak Positivstellensaetze for function algebras over ordered fields},
  pdfkeywords={Positivstellensatz, certificates of positivity, certificates of optimality, function algebras, ordered fields, o-minimality}
}
\newtheorem{theorem}{Theorem}[section]
\newtheorem{lemma}[theorem]{Lemma}
\newtheorem{corollary}[theorem]{Corollary}
\newtheorem{proposition}[theorem]{Proposition}

\theoremstyle{definition}
\newtheorem{definition}[theorem]{Definition}
\newtheorem{example}[theorem]{Example}
\newtheorem{remark}[theorem]{Remark}

\AddToHook{env/lemma/begin}{\crefalias{theorem}{lemma}}
\AddToHook{env/corollary/begin}{\crefalias{theorem}{corollary}}
\AddToHook{env/proposition/begin}{\crefalias{theorem}{proposition}}
\AddToHook{env/definition/begin}{\crefalias{theorem}{definition}}
\AddToHook{env/example/begin}{\crefalias{theorem}{example}}
\AddToHook{env/remark/begin}{\crefalias{theorem}{remark}}

\newcommand\N{\mathbb{N}}
\newcommand\Q{\mathbb{Q}}
\newcommand\R{\mathbb{R}}
\newcommand\Z{\mathbb{Z}}
\newcommand\calL{\mathcal{L}}
\newcommand\lh{\mathsf{lh}}
\newcommand\ReLU{\operatorname{ReLU}}
\newcommand\oF{\operatorname{oF}}
\newcommand\wP{\operatorname{wP}}
\newcommand\sP{\operatorname{sP}}
\newcommand\bR{\bm{R}}
\newcommand\T[1]{(T$#1$)}

\title[Axiomatic Positivstellens\"{a}tze]{Axiomatic Positivstellens\"{a}tze for Function Algebras over Ordered Fields}

\author{Nayoon Kim}
\address[Nayoon Kim, Allen Gehret, Shenyuan Ma, and Jakub Mare\v{c}ek]{%
  Czech Technical University in Prague\\
  Artificial Intelligence Center\\
  Charles Square 13\\
  Prague 2, Czech Republic}
\email[Nayoon Kim, corresponding author]{nayoon.kim@fel.cvut.cz}

\author{Allen Gehret}
\author{Shenyuan Ma}
\author{Jakub Mare\v{c}ek}

\subjclass[2020]{14P10, 03C64, 90C26, 26E40}
\keywords{Positivstellensatz, certificates of positivity, certificates of optimality, function algebras, ordered fields, o-minimality}

\date{\today}

\begin{document}

\begin{abstract}
Positivstellens\"{a}tze furnish algebraic certificates of positivity, and hence of lower bounds and of global optimality, but the classical ones concern polynomials, whereas many problems of current interest, among them the training of neural networks, involve functions that are neither polynomial nor smooth. We isolate an axiomatic core of the constructive Positivstellens\"{a}tze of Fischer. The setting consists of an ordered field expanded by five auxiliary function symbols, subject to a few first-order axioms, together with an algebra of functions on an arbitrary set, subject to closure axioms. In this setting we prove a strict and a weak Positivstellensatz in which the witnesses are obtained by substituting the data into fixed terms that depend only on the number of constraints. The framework recovers Fischer's theorems for definable $C^r$-functions, and yields new instances: algebras of definable functions with arbitrary auxiliary functions, such as the binary step function; algebras of continuous functions, for which we give simple sufficient criteria and auxiliary functions derived from common activation functions; and algebras of continuous functions with values in an ordered field that is not closed under square roots, such as $\Q$. We derive certificates of lower bounds and of global optimality, and show that their expanded term length is cubic in the number of constraints for the strict certificate and quadratic for the weak one.
\end{abstract}

\maketitle

\section{Introduction}\label{sec:intro}

A certificate of positivity for a function $g$ on a set $F=\{f_1\geq0,\ldots,f_k\geq0\}$ is an algebraic identity from which the positivity of $g$ on $F$ may be read off. For polynomials, the Positivstellens\"{a}tze of real algebraic geometry provide such certificates \cite{scheiderer2024course}, and they underlie the moment and sums-of-squares hierarchies of semidefinite relaxations for polynomial optimization \cite{lasserre2001global,parrilo2003semidefinite}. Many problems of current interest, however, are posed in terms of functions that are not polynomial. In the training of neural networks, for instance, the sigmoid activation is smooth but not semialgebraic, whereas $\ReLU$ is semialgebraic but not smooth, and most of the available certificates of optimality rely upon one property or the other. Certificates of other kinds have also been studied, among them branch-and-bound and cutting-plane proof systems \cite{huchette2026deep,NEURIPS2022_0b06c867,de2025truly} and convexifications \cite{NEURIPS2018_f2f44698,1011453498704,palma2024expressive,mao2024expressiveness,balauca2024gaussian}, and strong lower bounds are known for the size of branching and cutting-plane proofs \cite{dey2023lower,fleming2021power}. Semialgebraic certificates have been used to estimate the Lipschitz constants of ReLU networks \cite{NEURIPS2020_dea9ddb2} and to certify monotone deep equilibrium models \cite{NEURIPS2021_e3b21256}; their computational cost remains considerable, even when sparsity \cite{wang2021chordal,xue2022chordal,wang2024scalability} and nonnegativity \cite{mai2026tractable} are exploited.

We ask which \emph{S\"{a}tze} yield explicit certificates for functions that are not polynomial, and how such certificates may be constructed. Many proofs of the classical \emph{S\"{a}tze} use non-constructive arguments. An exception is the Positivstellensatz for differentiable functions of Fischer \cite{FischerPsatzDiff}, a constructive generalization of the Positivstellensatz for definable functions on o-minimal structures of Acquistapace, Andradas and Broglia \cite{AAB}. Fischer's construction is explicit, and it is this explicit structure that makes the present work possible. Two features of its formulation nevertheless make the construction difficult to reuse.
\begin{enumerate}[label=(\roman*)]
\item The theorems and proofs are given for (definable families of) definable $C^r$-functions over a definably complete real closed field. Variants for $C^r$-functions on Banach spaces and for Peano-differentiable functions are given, with the required proof modifications indicated.
\item The construction depends on the scalar field being real closed, chiefly because the square root plays a seemingly essential part.
\end{enumerate}
The first point suggests a single axiomatic framework of which all these settings are instances; the second raises the question of how freely functions such as the square root may be replaced.

\subsection{Contributions and organization}\label{sec:contributions}
We isolate the axiomatic core of Fischer's construction. In \cref{sec:setup} we introduce first-order theories $T_{\wP}\subseteq T_{\sP}$ for an ordered field expanded by five auxiliary function symbols $\sigma,\rho,\xi,\delta,\nu$, and closure axioms (AwP) and (AsP) for an algebra $A$ of functions on an arbitrary set $X$. In \cref{sec:main} we state a strict and a weak Positivstellensatz (\cref{thm:strict,thm:weak}): for each $k$ there are fixed terms, depending on $k$ alone, which, evaluated at $g,f_1,\ldots,f_k$, yield the witnesses of a certificate in every model of the axioms. The weak theorem does not require the algebra to be closed under the inversion of positive functions, and we show that this is a genuine weakening of the hypotheses. In \cref{sec:examples} we give instances that do not follow from \cite{FischerPsatzDiff}: algebras of definable functions with arbitrary auxiliary functions satisfying the scalar axioms, such as the binary step function; algebras of continuous functions, for which we give simple sufficient criteria together with auxiliary functions derived from common activation functions; and algebras of $\Q$-valued continuous functions, over a field that is not closed under square roots. Fischer's theorems are recovered as instances. \Cref{sec:optimality} derives certificates of lower bounds and of global optimality, and \cref{sec:complexity} computes the expanded length of the certificate terms. The proofs of the Positivstellens\"{a}tze are given in \cref{sec:proofs}, and the examples are verified in \cref{sec:verification}.

\section{Background}\label{sec:background}

A polynomial that is a sum of squares of polynomials is evidently nonnegative. Hilbert showed in 1888 that the converse fails, and Artin showed in 1927, in answer to Hilbert's seventeenth problem, that every nonnegative polynomial is a sum of squares of \emph{rational functions}. This motivates the study of certificates of positivity. Let $g,f_1,\ldots,f_k\in\R[X_1,\ldots,X_n]$ and $F=\bigcap_i\{f_i\geq0\}$. The \emph{preordering} $T(f_1,\ldots,f_k)$ is the smallest set of polynomials that contains all squares and each $f_i$ and is closed under addition and multiplication; its elements are nonnegative on $F$ by construction. The Positivstellensatz of Krivine and Stengle \cite{KrivinePsatz,StenglePsatz} states that $g\geq0$ on $F$ if and only if $(g^{2\mu}+a_1)g=a_2$ for some $\mu\in\N$ and $a_1,a_2\in T(f_1,\ldots,f_k)$. For $g>0$ on $F$, Schm\"{u}dgen's theorem \cite{SchmudgenPsatz} gives $g\in T(f_1,\ldots,f_k)$ when $F$ is compact; Putinar's theorem \cite{PutinarPsatz} gives $g=\sigma_0+\sum_i\sigma_if_i$ with sum of squares $\sigma_i$ when the quadratic module generated by the $f_i$ is Archimedean.

Beyond polynomials, Gamboa proved a Positivstellensatz for rings of continuous functions \cite{gamboa1987positivstellensatz}. Acquistapace, Andradas and Broglia extended the strict and weak Positivstellens\"{a}tze to definable $C^r$-functions in o-minimal expansions of real closed fields \cite{AAB}, and Fischer gave an explicit finite construction preserving definability uniformly in parameters \cite{FischerPsatzDiff}. Positivity certificates for functions that are not polynomial have also been obtained by other mechanisms \cite{dinh2021nichtnegativstellens,lasserre2010positivity,marshall2012positivstellensatze}. The present framework abstracts the formal mechanism of Fischer's construction: the auxiliary operations are encoded as function symbols of a first-order language, so that the Positivstellens\"{a}tze can be stated and proved independently of the underlying function algebra and scalar field.

\section{The axiomatic setting}\label{sec:setup}

The guiding example is the algebra $C^0(U,\R)$ of continuous functions on an open set $U\subseteq\R^n$, with the \emph{underlying set} $X=U$, the \emph{scalars} $\R$, and the \emph{algebra} $A=C^0(U,\R)$. To obtain further examples we allow $X$ to be an arbitrary set, the scalars to be a first-order structure $\bR=(R;\ldots)$ satisfying certain axioms, and $A$ to be a set of functions $X\to R$ satisfying certain closure axioms. We use the model-theoretic terminology of \cite[Appendix~B]{ADAMTT}. \emph{Throughout \cref{sec:setup,sec:main,sec:proofs}, $X$ denotes an arbitrary set.}

\subsection{Scalars, terms and algebras}

Let $\calL$ be a one-sorted first-order language and $\bR=(R;\ldots)$ an $\calL$-structure. Terms of $\calL$ define functions on $R^X$ pointwise.

\begin{definition}
Let $t(z_1,\ldots,z_n)$ be an $\calL$-term and $f_1,\ldots,f_n\in R^X$. Then $t(f_1,\ldots,f_n)\in R^X$ denotes the function $x\mapsto t^{\bR}(f_1(x),\ldots,f_n(x))$, where $t^{\bR}:R^n\to R$ is the interpretation of $t$ in $\bR$.
\end{definition}

Let $\calL_{\oF}:=\{\leq,0,1,-,+,\cdot,(\cdot)^{-1}\}$ be the language of ordered fields, and $T_{\oF}$ the $\calL_{\oF}$-theory whose models are the ordered fields with the usual interpretation of the symbols and the convention $0^{-1}:=0$. All languages below extend $\calL_{\oF}$, and all theories extend $T_{\oF}$. For $f:X\to R$ we write $\{f\geq0\}:=\{x\in X:f(x)\geq0\}$, and similarly $\{f>0\}$, $\{f=0\}$, and so on. The first two axioms on a set $A\subseteq R^X$ are the following.
\begin{enumerate}
\item[(A0)] $A$ contains the constant functions $0$ and $1$ and is closed under $-$, $+$ and $\cdot$.
\item[(A1)] If $f\in A$ and $f>0$ on $X$, then $f^{-1}\in A$.
\end{enumerate}
Under (A0) and (A1), $A$ is a $\Q$-algebra, which explains our use of the word \emph{algebra}; we use the word more broadly for sets of functions satisfying the closure axioms below. Axiom (A1) is imposed only where it is needed: it belongs to the strict package (AsP), but not to the weak package (AwP).

\begin{example}[A non-example]\label{ex:nonexample}
Let $n\geq1$ and $A=\R[X_1,\ldots,X_n]$, regarded as a set of functions $\R^n\to\R$. Then $A$ satisfies (A0) but not (A1): the polynomial $f=1+X_1^2+\cdots+X_n^2$ is positive on $\R^n$, whereas $f^{-1}$, being bounded and not constant, is not a polynomial. Moreover, $A$ fails (A2) for every $\xi:\R\to\R$ satisfying \T{\xi1} and \T{\xi2} below. Indeed, (A2) would imply $\xi(X_1)\in A$. Restricting to $(t,0,\ldots,0)$ would then make $\xi$ a univariate polynomial, contradicting its vanishing on $(-\infty,0]$ and positivity on $(0,\infty)$. Then neither of the theorems below applies to $\R[X_1,\ldots,X_n]$.
\end{example}

\subsection{The auxiliary functions and the scalar theories}

We expand the field of scalars by five function symbols, whose intended roles are as follows.
\begin{itemize}
\item $\sigma$, unary, produces nonnegative values (for instance $x^2$, $|x|$ or $x^{2d}$);
\item $\rho$, unary, is a right inverse of $\sigma$ on $[0,\infty)$ (for instance $\sqrt{x}$);
\item $\xi$, unary, is a generalised positive part, or gate (for instance $\ReLU(x):=\max(0,x)$, or $\exp(-1/x)$ for $x>0$ and $0$ otherwise);
\item $\delta$, unary, is strictly positive and at least twice its argument (for instance $2\ReLU(x)+1$, or $2\sqrt{1+x^2}$);
\item $\nu$, binary, is a partial substitute for the maximum: it is positive whenever either argument is positive, and bounded by its second argument whenever the first is nonpositive and the second positive (for instance $\max(x,y)$ or $\sqrt{x^2+y^2}+x$).
\end{itemize}
We put $\calL_{\wP}:=\calL_{\oF}\cup\{\sigma,\rho,\xi\}\subseteq\calL_{\sP}:=\calL_{\wP}\cup\{\delta,\nu\}$. Let $T_{\wP}$ be the $\calL_{\wP}$-theory extending $T_{\oF}$ whose models satisfy, for all $a,b\in R$,
\begin{enumerate}
\item[\T{\xi1}] $\xi(a)\geq0$;
\item[\T{\xi2}] $\xi(a)>0$ if and only if $a>0$;
\item[\T{\sigma1}] $\sigma(a)\geq0$;
\item[\T{\sigma2}] if $a,b\geq0$, then $\sigma(ab)=\sigma(a)\sigma(b)$;
\item[\T{\rho}] if $a\geq0$, then $\rho(a)\geq0$;
\item[\T{\sigma\rho}] if $a\geq0$, then $\sigma(\rho(a))=a$;
\end{enumerate}
and let $T_{\sP}$ be the $\calL_{\sP}$-theory extending $T_{\wP}$ whose models satisfy, in addition,
\begin{enumerate}
\item[\T{\delta}] $0<\delta(a)$ and $2a\leq\delta(a)$;
\item[\T{\nu1}] if $a>0$ or $b>0$, then $\nu(a,b)>0$;
\item[\T{\nu2}] if $a\leq0$ and $b>0$, then $\nu(a,b)\leq b$.
\end{enumerate}

\begin{example}[The main example]\label{ex:main}
Let $X=U\subseteq\R^n$ be open and $A=C^0(U,\R)$, and expand the real field to an $\calL_{\sP}$-structure by
\[
\sigma(x)=x^2,\qquad\rho(x)=\sqrt{|x|},\qquad\xi(x)=\ReLU(x)^2,\qquad\delta(x)=2\sqrt{1+x^2},
\]
and $\nu(x,y)=\sqrt{x^2+y^2}+x$.
These functions are shown in \cref{fig:aux}. The structure is a model of $T_{\sP}$, and $A$ satisfies all the closure axioms below (\cref{cor:mainexample}).
\end{example}

\begin{figure}[t]
\centering
\input{figures/auxiliary_functions.tex}
\caption{The auxiliary functions $\sigma,\rho,\xi,\delta,\nu$ of the main example (\cref{ex:main}). Since $\nu(x,y)$ depends on $y$ only through $y^2$, the last panel shows $\nu(\cdot,y)$ for $y=0$ and $|y|=1,2$. Each curve is positive where $x>0$ or $y\neq0$, in accordance with \T{\nu1}, and for $y>0$ it stays at or below the dotted level $y$ on $x\leq0$, in accordance with \T{\nu2}.}
\label{fig:aux}
\end{figure}

\subsection{The closure axioms}

For the strict Positivstellensatz we require, for all $f,g\in A$:
\begin{enumerate}
\item[(A2)] $\xi(f)\in A$;
\item[(A3s)] if $f>0$ on $X$, then $\sigma(f),\rho(f)\in A$;
\item[(A4s)] if $\{f\geq0\}\subseteq\{g\neq0\}$, then $\xi(f)\cdot g^{-1}\in A$;
\item[(A$\delta$)] $\delta(f)\in A$;
\item[(A$\nu$)] if $\{f\leq0\}\subseteq\{g>0\}$, then $\nu(f,g)\in A$;
\item[(AsP)] the conjunction of (A0), (A1), (A2), (A3s), (A4s), (A$\delta$) and (A$\nu$).
\end{enumerate}
For the weak Positivstellensatz we require, for all $f,g,h\in A$:
\begin{enumerate}
\item[(A3w)] if $f\geq0$ on $X$, then $\sigma(f),\ \rho(f)\cdot\xi(f)\in A$;
\item[(A4w)] $\xi(-f)\cdot f^{-1}\in A$;
\item[(A5w)] if $\{h\geq0\}\subseteq\{g\geq0\}$, then $\xi(g^2)\cdot\xi(g+4h)\cdot(g+h)^{-1}\in A$;
\item[(AwP)] the conjunction of (A0), (A2), (A3w), (A4w) and (A5w).
\end{enumerate}
The axioms (A4s), (A4w) and (A5w) concern quotients whose denominators may vanish; with the convention $0^{-1}=0$, the gates in the numerators make these quotients regular in the examples. For instance, when $\xi(t)=\ReLU^2(t)$, we have $\xi(-f)f^{-1}=\min(f,0)=-\ReLU(-f)$, which is continuous whenever $f$ is.

\section{Main results}\label{sec:main}

\begin{theorem}[Strict Positivstellensatz; cf.\ {\cite[Theorem~1.1]{FischerPsatzDiff}}]\label{thm:strict}
For each $k\geq0$ there are $\calL_{\sP}$-terms $\tilde t_0,\ldots,\tilde t_k$ in the variables $z_0,\ldots,z_k$ with the following property. Let $\bR\models T_{\sP}$, let $A\subseteq R^X$ satisfy \textup{(AsP)}, and let $g,f_1,\ldots,f_k\in A$ satisfy $\bigcap_{i}\{f_i\geq0\}\subseteq\{g>0\}$. Then each $t_i:=\tilde t_i(g,f_1,\ldots,f_k)$ lies in $A$ and is strictly positive on $X$, and
\[
g\ =\ \sigma(t_0)+\textstyle\sum_{i=1}^k\sigma(t_i)\,f_i .
\]
\end{theorem}

\begin{theorem}[Weak Positivstellensatz; cf.\ {\cite[Theorem~1.2]{FischerPsatzDiff}}]\label{thm:weak}
For each $k\geq0$ there are $\calL_{\wP}$-terms $\tilde t_{-1},\tilde t_0,\ldots,\tilde t_k$ in the variables $z_0,\ldots,z_k$ with the following property. Let $\bR\models T_{\wP}$, let $A\subseteq R^X$ satisfy \textup{(AwP)}, and let $g,f_1,\ldots,f_k\in A$ satisfy $\bigcap_i\{f_i\geq0\}\subseteq\{g\geq0\}$. Then $p:=\tilde t_{-1}(g,f_1,\ldots,f_k)$ and $t_i:=\tilde t_i(g,f_1,\ldots,f_k)$, $0\leq i\leq k$, are nonnegative functions in $A$, and
\[
\sigma(p)\,g\ =\ \sigma(t_0)+\textstyle\sum_{i=1}^k\sigma(t_i)\,f_i ,\qquad \{p=0\}\ =\ \{g=0\}.
\]
\end{theorem}

The terms $\tilde t_i$ are \emph{universal}: they depend only on $k$, and not on $\bR$, $A$, $X$ or the functions concerned; only the verification of the axioms depends on the instance. In the main example (\cref{ex:main}) the theorems state that if $g,f_1,\ldots,f_k$ are continuous on $U$ and $F:=\bigcap_i\{f_i\geq0\}\subseteq\{g>0\}$, then $g=t_0^2+\sum_it_i^2f_i$ with continuous and strictly positive $t_i$; and that if $F\subseteq\{g\geq0\}$, then $p^2g=t_0^2+\sum_it_i^2f_i$ with continuous and nonnegative $p,t_i$ such that $\{p=0\}=\{g=0\}$. The proofs are given in \cref{sec:proofs}.

\begin{remark}[The form of the certificates]\label{rem:form}
The certificates of \cref{thm:strict} have the shape of Putinar's representation, but with a single value of $\sigma$ as the multiplier of each constraint, rather than a sum of squares, and with no products of constraints, as would appear in Schm\"{u}dgen's or Krivine and Stengle's representations. This economy is possible because the algebra is closed under the auxiliary functions: in the main example, for instance, a sum of squares of positive continuous functions is the square of a positive continuous function. No compactness or Archimedean condition is required, and the multipliers need not be polynomials.
\end{remark}

\begin{remark}\label{rem:A1}
Unlike \cref{thm:strict}, \cref{thm:weak} does not require $A$ to be closed under the inversion of everywhere positive functions: after pointwise cancellation, only the regularised quotients of (A4w) and (A5w) are needed. This is a genuine weakening of the hypotheses, since there are algebras satisfying (AwP) but not (A1) (\cref{prop:A1indep}).
\end{remark}

\subsection{The universal terms}\label{sec:terms}

The terms are extracted from Fischer's construction \cite{FischerPsatzDiff}. We write $s/t$ for $s\cdot t^{-1}$ and $s-t$ for $s+(-t)$. Sums and products are associated to the left, a natural number $m\geq2$ denotes the term $1+\cdots+1$ with $m$ summands.

\begin{definition}[The strict construction]\label{def:strictterms}
Put $\tilde\varepsilon(z_0,z_1,z_2):=\nu(z_0,z_1)/\delta(z_2)$ and, for $1\leq i\leq k$,
\begin{align*}
\tilde\psi&:=\textstyle\sum_{i=1}^k\xi(-z_i)\cdot z_i, &
\tilde\varepsilon_i&:=\tilde\varepsilon\bigl(z_0,-\tilde\psi/k,z_i\bigr),\\
\tilde s_i&:=\rho\bigl(\xi(-z_i)+\tilde\varepsilon_i\bigr), &
\tilde h&:=\textstyle\sum_{i=1}^k\sigma(\tilde s_i)\cdot z_i,\\
\tilde\varphi_1&:=\xi(z_0+4\tilde h), &
\tilde\varphi_2&:=\xi(-(z_0+\tilde h)),\\
\tilde\varphi_3&:=\xi(-z_0\tilde h), &
\tilde\varphi&:=\tilde\varphi_1+\tilde\varphi_2+\tilde\varphi_3,\\
\tilde w&:=\Bigl(\frac{z_0\tilde\varphi_1}{z_0+\tilde h}+\frac{(z_0+\tilde h)\tilde\varphi_2}{\tilde h}+\tilde\varphi_3\Bigr)\Big/\tilde\varphi, &
\tilde u&:=z_0-\tilde w\cdot\tilde h,
\end{align*}
and finally $\tilde t_0:=\rho(\tilde u)$ and $\tilde t_i:=\rho(\tilde w)\cdot\tilde s_i$.
\end{definition}

\begin{definition}[The weak construction]\label{def:weakterms}
Put
\begin{align*}
\tilde h&:=\textstyle\sum_{i=1}^k\sigma(\xi(-z_i))\cdot z_i, &
\tilde\varphi_1&:=\xi(z_0+4\tilde h),\\
\tilde\varphi_2&:=\xi(-(z_0+\tilde h)), &
\tilde\varphi_3&:=\xi(-z_0\tilde h),\\
\tilde\varphi&:=\tilde\varphi_1+\tilde\varphi_2+\tilde\varphi_3, &
\tilde q&:=\tilde\varphi\cdot\xi(z_0^2)\cdot\bigl(\xi(-\tilde h)+\xi(z_0)\xi(z_0+\tilde h)\bigr),\\
\tilde\omega_1&:=\tilde\varphi_3\tilde q/\tilde\varphi, &
\tilde\omega_2&:=\tilde q\tilde\varphi_2(z_0+\tilde h)/(\tilde h\tilde\varphi),\\
\tilde\omega_3&:=\tilde q\tilde\varphi_1z_0/((z_0+\tilde h)\tilde\varphi), &
\tilde\omega&:=\tilde\omega_1+\tilde\omega_2+\tilde\omega_3,
\end{align*}
and $\tilde u:=\tilde qz_0-\tilde\omega\tilde h$; then $\tilde\psi_1:=\xi(\tilde u)$, $\tilde\Psi_1:=\xi(\tilde u)\rho(\tilde u)$, $\tilde\psi_2:=\xi(\tilde\omega)$, $\tilde\Psi_2:=\xi(\tilde\omega)\rho(\tilde\omega)$, $\tilde\psi_3:=\xi(\tilde q)$, $\tilde\Psi_3:=\xi(\tilde q)\rho(\tilde q)$, and finally
\[
\tilde t_{-1}:=\tilde\psi_1\tilde\psi_2\tilde\Psi_3,\qquad \tilde t_0:=\tilde\Psi_1\tilde\psi_2\tilde\psi_3,\qquad \tilde t_i:=\tilde\psi_1\tilde\Psi_2\tilde\psi_3\,\xi(-z_i)\quad(1\leq i\leq k).
\]
\end{definition}

\subsection{The idea of the proofs}\label{sec:idea}
In both cases the constraints are first combined into a single function $h$. In the strict case (\cref{fig:signregions}) $h=\sum_i(\xi(-f_i)+\varepsilon_i)f_i$, where $\xi(-f_i)$ detects the violation of the $i$-th constraint and the small positive corrections $\varepsilon_i$ ensure that $F\subseteq\{h\geq0\}\subseteq\{g>0\}$. From then on only the pair $(g,h)$ matters. The gates $\varphi_1,\varphi_2,\varphi_3$ are positive on three sets that cover $X$; on each of them there is an evident multiplier $w_j>0$ with $g-w_jh>0$, namely $g/(g+h)$, $(g+h)/h$ and $1$, and $w$ glues these together, so that $u=g-wh>0$ and $g=\sigma(\rho(u))+\sum_i\sigma(\rho(w)s_i)f_i$. In the weak case the three sets miss the common zero set of $g$ and $h$, the divisions are regularised by the factor $q$, which vanishes there, and the nonnegative functions $u$, $\omega$ and $q$ are written as quotients of values of $\sigma$ in the manner of Hilbert's seventeenth problem (\cref{lem:H17}).

\begin{figure}[t]
\centering
\begin{tikzpicture}[scale=1.25,font=\footnotesize]
  \fill[blue!10] (0,0) -- (3,-0.75) -- (3,2) -- (0,2) -- cycle;
  \fill[red!12] (-2,0) -- (0,0) -- (2,-2) -- (-2,-2) -- cycle;
  \fill[green!25,opacity=0.6] (0,0) -- (3,0) -- (3,-2) -- (0,-2) -- cycle;
  \fill[gray!30] (-2,0) -- (0,0) -- (0,2) -- (-2,2) -- cycle;
  \draw[blue!70!black,dashed] (0,0) -- (3,-0.75) node[right] {$g+4h=0$};
  \draw[red!70!black,dashed] (0,0) -- (2,-2) node[below right] {$g+h=0$};
  \draw[->] (-2.2,0) -- (3.25,0) node[right] {$g$};
  \draw[->] (0,-2.2) -- (0,2.25) node[above] {$h$};
  \node[align=center] at (-1,1) {excluded:\\ $g\leq0\leq h$};
  \node[align=center,blue!50!black] at (1.5,1.15) {$U_1$: $w_1=\dfrac{g}{g+h}$,\\[2pt] $g-w_1h=\dfrac{g^2}{g+h}$};
  \node[align=center,red!50!black] at (-1.05,-1.2) {$U_2$: $w_2=\dfrac{g+h}{h}$,\\[2pt] $g-w_2h=-h$};
  \node[align=center,green!30!black] at (2.42,-1.1) {$U_3$: $w_3=1$,\\ $g-w_3h=g-h$};
\end{tikzpicture}
\caption{The sign regions in the $(g,h)$-plane used in the proof of \cref{thm:strict}. The values $(g(x),h(x))$ avoid the grey quadrant (\cref{lem:hstrict}), and the remaining plane is covered by $U_1=\{g+4h>0\}$, $U_2=\{g+h<0\}$ and $U_3=\{g>0>h\}$, on which the gates $\varphi_1,\varphi_2,\varphi_3$ are positive. On each $U_j$ the local multiplier $w_j$ is positive and makes $g-w_jh$ positive; the multiplier $w$ is their average weighted by the gates. In the weak case $h\leq0$, and only the origin is missed by the three regions.}
\label{fig:signregions}
\end{figure}

\begin{example}[The case $k=0$]\label{ex:k0}
For $k=0$ the hypothesis reduces to $g>0$, respectively $g\geq0$, on $X$. In the strict construction $\tilde\psi$ and $\tilde h$ are the empty sum $0$. Substituting the universal term $\tilde{t}_0$ into the certificate $g=\sigma(t_0)$ gives
\begin{multline*}
g=\sigma\Bigl(\rho\Bigl(g-\Bigl[\frac{g\,\xi(g+4\cdot0)}{g+0}+\frac{(g+0)\,\xi(-(g+0))}{0}+\xi(-g\cdot0)\Bigr]\\
\cdot\bigl(\xi(g+4\cdot0)+\xi(-(g+0))+\xi(-g\cdot0)\bigr)^{-1}\cdot0\Bigr)\Bigr),
\end{multline*}
which reduces to $g=\sigma(\rho(g))$ with $t_0=\rho(g)>0$. In the weak construction $\tilde h=0$, and using $0^{-1}=0$ and $\xi(g)\xi(-g)=0$ gives $q=\xi(g^2)\xi(g)^3$, $\omega=q$ and $u=qg$. Consequently, the certificate $\sigma(p)g=\sigma(t_0)$ is given by
\[
p=\xi(qg)\,\xi(q)^2\rho(q),\qquad t_0=\xi(qg)\rho(qg)\,\xi(q)^2 .
\]
The identity alone certifies $g\geq0$ only on $\{p>0\}$; together with $\{p=0\}=\{g=0\}$ it certifies $g\geq0$ on $X$.
\end{example}

\section{Examples}\label{sec:examples}

We now give instances of \cref{thm:strict,thm:weak}. The proofs of the claims made here are given in \cref{sec:verification}.

\subsection{Definable functions with arbitrary auxiliary functions}\label{sec:definable}

Outside the continuous setting one may wish to take for $\xi$ the binary step function, $\xi(x)=1$ for $x>0$ and $\xi(x)=0$ otherwise. More generally, the auxiliary functions may be chosen with complete freedom, subject only to the scalar axioms, provided that the algebra is chosen appropriately. Let $\mathcal{L}\supseteq\mathcal{L}_{\sP}$ be a language and $\bR$ an $\mathcal{L}$-structure satisfying $T_{\sP}$. Let $X\subseteq R^n$ be a nonempty definable set and $A$ the set of all definable functions $X\to R$; here and below, \emph{definable} means definable with parameters in the structure.

\begin{corollary}\label{cor:definable}
The algebra $A$ satisfies \textup{(AsP)} and \textup{(AwP)}.
\end{corollary}

For instance, if $\bR$ is an o-minimal expansion of the real field in which the auxiliary functions are definable (which allows $\xi$ to be the binary step function or $\ReLU$), then every function in $A$ is piecewise $C^r$ for each fixed $r\in\mathbb{N}$, by the $C^r$-cell decomposition theorem \cite[Chapter~7, \S3]{driesTameTopologyOminimal1998}. At the other extreme, if every subset of every $\R^n$ is definable in $\bR$, then $A=\R^X$.

\subsection{Continuous functions}\label{sec:continuous}

Let $\bR$ be a model of $T_{\oF}$, equip $R$ with the order topology and $R^n$ with the product topology, and let $X$ be a topological space. We say that $f:R\to R$ is \emph{strictly dominated by} $g:R\to R$ at $0^+$, and write $f\prec g$ at $0^+$, if for every $\varepsilon>0$ there is $\eta>0$ with $|f(t)|<\varepsilon|g(t)|$ for all $t\in(0,\eta)$. For algebras of continuous functions, continuity conditions on the auxiliary functions suffice in the strict case; the weak case additionally assumes $\xi(t)\prec t$ at $0^+$.

\begin{proposition}\label{prop:C0}
Let $A=C^0(X,R)$.
\begin{enumerate}
\item Suppose that $\bR$ is expanded to a model of $T_{\sP}$ in which $\xi$ and $\delta$ are continuous on $R$, $\sigma$ and $\rho$ are continuous on $(0,+\infty)$, and $\nu$ is continuous on $\{(a,b)\in R^2:a\leq0\Rightarrow b>0\}$. Then $A$ satisfies \textup{(AsP)}.
\item Suppose that $\bR$ is expanded to a model of $T_{\wP}$ in which $\xi$ is continuous with $\xi(t)\prec t$ at $0^+$, and $\sigma$ and $t\mapsto\xi(t)\rho(t)$ are continuous on $[0,+\infty)$. Then $A$ satisfies \textup{(AwP)}; moreover, $\rho$ is continuous on $[0,+\infty)$, so that $\rho(f)\in A$ for every nonnegative $f\in A$.
\end{enumerate}
\end{proposition}

The following lemmas give two families of auxiliary functions satisfying these sufficient conditions.

\begin{lemma}[Auxiliary functions of ReLU type]\label{lem:relu}
Let $s\geq1$ be an integer, and put $\xi=\ReLU^s$, $\sigma=\rho=\ReLU$, $\delta=2\ReLU+1$ and $\nu=\max$. Then the axioms of $T_{\sP}$ hold, all five functions are continuous, and $\xi(t)\prec t$ at $0^+$ if $s\geq2$.
\end{lemma}

\begin{lemma}[Auxiliary functions of square-root type]\label{lem:sqrt}
Suppose that every nonnegative element of $R$ has a square root, and put $\sigma(x)=x^2$, $\rho(x)=\sqrt{|x|}$, $\delta(x)=2\sqrt{1+x^2}$ and $\nu(x,y)=\sqrt{x^2+y^2}+x$. Then the axioms \T{\sigma1}--\T{\sigma\rho}, \T{\delta} and \T{\nu1}--\T{\nu2} hold, and $\sigma,\rho,\delta,\nu$ are continuous on $R$ and $R^2$.
\end{lemma}

\begin{corollary}\label{cor:mainexample}
In the main example (\cref{ex:main}), the real field with the stated auxiliary functions is a model of $T_{\sP}$, and $C^0(U,\R)$ satisfies \textup{(AsP)} and \textup{(AwP)}.
\end{corollary}

The main example itself is covered by \cite{FischerPsatzDiff}; \cref{prop:C0} shows how freely its auxiliary functions may be varied. The gate may be any power $\xi=\ReLU^\alpha$ with $\alpha>0$ for (AsP), and with $\alpha>1$ for (AwP); for $0<\alpha\leq1$ the condition $\xi(t)\prec t$ at $0^+$ fails. The gate may also be derived from a common activation function: if $\Phi:\R\to\R$ is continuous, differentiable at $0$ and strictly increasing on $[0,+\infty)$, then
\[
\xi_\Phi(t):=(\Phi(t)-\Phi(0))^2\ \ \text{for }t>0,\qquad \xi_\Phi(t):=0\ \ \text{for }t\leq0,
\]
satisfies \T{\xi1}, \T{\xi2} and the conditions on $\xi$ in both parts of \cref{prop:C0}, since $\xi_\Phi(t)/t=(\Phi(t)-\Phi(0))^2/t\to\Phi'(0)^2\cdot0=0$ as $t\downarrow0$. Taking for $\Phi$ the GELU, the sigmoid or softplus \cite{kunc2024three} gives, for $t>0$, the gates $\mathsf{GELU}(t)^2$, $(\mathsf{sigmoid}(t)-\frac12)^2$ and $(\mathsf{softplus}(t)-\log2)^2$ shown in \cref{fig:gates}.

\begin{figure}[t]
\centering
\input{figures/activation_gates.tex}
\caption{Gates $\xi_\Phi$ (solid) derived from the activation functions $\Phi=\mathsf{GELU}$, $\mathsf{sigmoid}$ and $\mathsf{softplus}$, together with $\Phi(t)-\Phi(0)$ (dashed).}
\label{fig:gates}
\end{figure}

\subsection{Rational-valued continuous functions}\label{sec:rational}

In \cite{FischerPsatzDiff} the field of scalars is always real closed, chiefly because the implicit $\rho$, $\delta$ and $\nu$ are defined by means of $\sqrt{\cdot}$. The following example makes sense in any ordered field; for definiteness we state it for $\Q$. Expand $(\Q;\leq,+,\cdot)$ to an $\calL_{\sP}$-structure $\bm Q$ by
\[
\sigma(x)=\rho(x)=\ReLU(x),\qquad\xi(x)=\ReLU(x)^2,\qquad\delta(x)=2\ReLU(x)+1,
\]
and $\nu(x,y)=\max(x,y)$; equip $\Q$ with the order topology, and let $X$ be a topological space.

\begin{corollary}\label{cor:rational}
$\bm Q$ is a model of $T_{\sP}$, and $C^0(X,\Q)$ satisfies \textup{(AsP)} and \textup{(AwP)}.
\end{corollary}

With these interpretations, $\sigma$ is the identity on nonnegative values, so that the certificates take the form $g=t_0+\sum_it_if_i$ with $t_i>0$, and $p\,g=t_0+\sum_it_if_i$ with $p,t_i\geq0$ and $\{p=0\}=\{g=0\}$. This is a model of exact rational arithmetic; it is not a model of floating-point or fixed-point arithmetic, in which rounding, overflow and inexact inversion would call for a semantics that accounts for errors.

\subsection{Fischer's setting}\label{sec:fischer}

Fix $r\in\N\cup\{\infty\}$ and $d,k\geq1$. Let $\calL\supseteq\calL_{\oF}$ with $\calL\cap\calL_{\sP}=\calL_{\oF}$, and let $\bR$ be an $\calL$-structure that is a model of $T_{\oF}$ and is \emph{definably complete}: every nonempty definable subset of $R$ that is bounded above has a supremum. Then $R$ is real closed, and the elementary differential calculus is available for definable functions \cite[\S4]{Miller_Expansions_DLO_IVP}, so that one may speak of definable $C^r$-manifolds and definable $C^r$-functions. If $r=\infty$, we assume in addition that $\bR$ defines a $C^1$-function $\exp$ with $\exp'=\exp$, $\exp(0)=1$ and $\exp(x+y)=\exp(x)\exp(y)$. Let $M\subseteq R^n$ be a definable $C^r$-manifold and $S\subseteq R^m$ a definable set. A \emph{definable family of definable $C^r$-functions} $M\to R$ is a definable function $f:M\times S\to R$ such that each $f_s:=f(\cdot,s)$ is $C^r$. Let $g_s,f_{1,s},\ldots,f_{k,s}$ be such families, and put $F_s:=\bigcap_i\{f_{i,s}\geq0\}$.

\begin{corollary}[{\cite[Theorem~1.1]{FischerPsatzDiff}}]\label{cor:fischer-strict}
If $g_s>0$ on $F_s$ for every $s\in S$, then there are definable families $v_{0,s},\ldots,v_{k,s}$ of definable $C^r$-functions $M\to R$, strictly positive on $M$, with $g_s=v_{0,s}^{2d}+\sum_iv_{i,s}^{2d}f_{i,s}$ for every $s\in S$.
\end{corollary}

\begin{corollary}[{\cite[Theorem~1.2]{FischerPsatzDiff}}]\label{cor:fischer-weak}
If $g_s\geq0$ on $F_s$ for every $s\in S$, then there are definable families $p_s,v_{0,s},\ldots,v_{k,s}$ of definable $C^r$-functions $M\to R$, nonnegative on $M$, with $p_s^{2d}g_s=v_{0,s}^{2d}+\sum_iv_{i,s}^{2d}f_{i,s}$ and $\{p_s=0\}=\{g_s=0\}$ for every $s\in S$.
\end{corollary}

Fischer states the weaker conclusion $\{p_s=0\}\subseteq\{g_s=0\}$; the equality follows from both his construction and ours.

\section{Certificates of lower bounds and of optimality}\label{sec:optimality}

The Positivstellens\"{a}tze admit the following reformulation in terms of optimization.

\begin{corollary}\label{cor:optimality}
Let $g,f_1,\ldots,f_k\in A$, put $F:=\bigcap_i\{f_i\geq0\}$, fix $L\in R$, and assume that $A$ contains all constant functions $X\to R$.
\begin{enumerate}
\item Suppose that $\bR\models T_{\sP}$ and that $A$ satisfies \textup{(AsP)}. The following are equivalent: \\
\textup{(S1)} there is $\eta>0$ with $F\subseteq\{g>L+\eta\}$; \\
\textup{(S2)} there are $\eta>0$ and strictly positive $t_0,\ldots,t_k\in A$ with 
\[
g-(L+\eta)=\sigma(t_0)+\sum_i\sigma(t_i)f_i.
\]
\item Suppose that $\bR\models T_{\wP}$ and that $A$ satisfies \textup{(AwP)}. The following are equivalent: \\
\textup{(W1)} $F\subseteq\{g\geq L\}$; \\
\textup{(W2)} there are nonnegative $p,t_0,\ldots,t_k\in A$ with 
\[
\sigma(p)(g-L)=\sigma(t_0)+\sum_i\sigma(t_i)f_i
\]
and $\{p=0\}=\{g=L\}$.
\end{enumerate}
In both cases the witnesses may be obtained by substituting $g-(L+\eta)$, respectively $g-L$, together with $f_1,\ldots,f_k$, into the universal terms. If $F\neq\varnothing$ and $g^\star:=\inf_{x\in F}g(x)$ exists in $R$, then \textup{(S1)} is equivalent to $L<g^\star$ and \textup{(W1)} to $L\leq g^\star$. In particular, a point $x^\star\in F$ is a global minimiser of $g$ on $F$ if and only if \textup{(W2)} holds with $L=g(x^\star)$.
\end{corollary}

\begin{proof}
Since $A$ contains the constants, $g-(L+\eta)$ and $g-L$ lie in $A$, and the implications (S1)$\Rightarrow$(S2) and (W1)$\Rightarrow$(W2) are \cref{thm:strict,thm:weak}. Conversely, assume (S2) and let $x\in F$. Then $f_i(x)\geq0$, $\sigma(t_i(x))\geq0$ by \T{\sigma1}, and $\sigma(t_0(x))>0$ by \cref{lem:sigmarho}, since $t_0(x)>0$; hence $g(x)-(L+\eta)>0$. Assume (W2), let $x\in F$, and suppose that $g(x)<L$. Then $p(x)\neq0$, so $p(x)>0$ and $\sigma(p(x))>0$ by \cref{lem:sigmarho}; thus the left-hand side of the identity is negative at $x$, whereas the right-hand side is nonnegative, a contradiction. For the last assertions, (S1) gives $g^\star\geq L+\eta>L$; conversely, if $L<g^\star$, then $\eta:=(g^\star-L)/2$ satisfies $g\geq g^\star>L+\eta$ on $F$. The equivalence of (W1) with $L\leq g^\star$ is immediate, and $x^\star\in F$ is a global minimiser if and only if (W1) holds with $L=g(x^\star)$.
\end{proof}

Thus a lower bound strictly below the infimum is certified by an identity whose multipliers $\sigma(t_i)$ need not be constant, and the global optimality of a feasible point is certified by an identity together with a condition on a zero set. In the main example, $g,f_1,\ldots,f_k$ may be arbitrary continuous functions, for instance functions computed by a neural network with continuous activations, and the auxiliary functions used to build the certificate need not coincide with those activations. No compactness, Archimedean condition or bound on degrees is required.

\subsection{An explicit certificate over the rationals}\label{sec:roundingexample}

The universal terms produce explicit certificates also in settings that are neither continuous nor real closed. Let $Q(x):=\lfloor x+\frac12\rfloor\in\Z\subseteq\Q$ denote rounding to the nearest integer, and consider
\[
\inf_{x\in\Q}Q(x)\qquad\text{subject to}\qquad f(x):=x-1\geq0 .
\]
The feasible point $x^\star=1$ has $Q(x^\star)=1$. Take the auxiliary functions of \cref{sec:rational}, so that $\sigma=\rho=\ReLU$ and $\xi=\ReLU^2$ on $\Q$, and let $A$ be the algebra of all functions $\Q\to\Q$ definable in $(\Q;<,+,\cdot,Q)$, which satisfies (AwP) by \cref{cor:definable}. Then $Q-1,f\in A$, and $Q\geq1$ wherever $f\geq0$. Substituting $g=Q-1$ and $f$ into the weak terms of \cref{def:weakterms} with $k=1$, and simplifying, one obtains nonnegative $p,t_0,t_1\in A$ given by
\begin{align*}
p(x)&=\begin{cases}(1-Q(x))^{28}(1-x)^{42}\bigl(1-Q(x)+(1-x)^3\bigr)^{16},&x<\frac12,\\ 0,&\frac12\leq x<\frac32,\\ (Q(x)-1)^{72},&x\geq\frac32,\end{cases}\\
t_1(x)&=\begin{cases}(1-Q(x))^{28}(1-x)^{41}\bigl(1-Q(x)+(1-x)^3\bigr)^{17},&x<\frac12,\\ 0,&x\geq\frac12,\end{cases}
\end{align*}
and $t_0(x)=(1-x)^3p(x)$ for $x<\frac12$, $t_0(x)=(Q(x)-1)\,p(x)$ for $x\geq\frac32$, and $t_0(x)=0$ otherwise. Since $\sigma$ is the identity on nonnegative values, the certificate reads $p\,(Q-1)=t_0+t_1\,(x-1)$. It may be checked region by region: for $x<\frac12$ it reduces to $(1-x)^3-\bigl(1-Q+(1-x)^3\bigr)=-(1-Q)$, for $x\geq\frac32$ to $(Q-1)^{73}=(Q-1)^{73}$, and in between both sides vanish. Moreover $\{p=0\}=[\frac12,\frac32)\cap\Q=\{Q=1\}$, since $1-Q(x)\geq1$ and $1-x>0$ for $x<\frac12$. By \cref{cor:optimality}, $x^\star=1$ is a global minimiser, and the certificate never leaves $\Q$. We have also confirmed by evaluating the universal terms in exact rational arithmetic that they coincide with these closed forms. The high exponents are typical of the universal construction.

\section{The length of the universal terms}\label{sec:complexity}

Following \cite[Appendix~B]{ADAMTT}, we regard $\calL$-terms as admissible words over the alphabet consisting of the function symbols of $\calL$ and the variables $z_0,z_1,\ldots$. By unique readability \cite[B.3.1]{ADAMTT}, the length $\lh(t)$ of a term $t$ satisfies $\lh(z_i)=\lh(0)=\lh(1)=1$, $\lh(\mathfrak f\,t)=1+\lh(t)$ for $\mathfrak f\in\{-,(\cdot)^{-1},\sigma,\rho,\xi,\delta\}$, and $\lh(s+t)=\lh(s\cdot t)=\lh(\nu(s,t))=1+\lh(s)+\lh(t)$. Consequently, for $k\geq1$, $\lh(\sum_{i=1}^kt_i)=k-1+\sum_i\lh(t_i)$, and the numeral $k$ has length $2k-1$.

\begin{proposition}\label{prop:length}
Let $k\geq1$. The term on the right-hand side of \cref{thm:strict} has length
\[
\lh\Bigl(\sigma(\tilde t_0)+\textstyle\sum_{i=1}^k\sigma(\tilde t_i)\cdot z_i\Bigr)\ =\ 72k^3+232k^2+224k+51 ,
\]
and the terms on the left- and right-hand sides of \cref{thm:weak} have lengths
\[
\lh\bigl(\sigma(\tilde t_{-1})\cdot z_0\bigr)=532k+467,\qquad
\lh\Bigl(\sigma(\tilde t_0)+\textstyle\sum_{i=1}^k\sigma(\tilde t_i)\cdot z_i\Bigr)=707k^2+1370k+649 .
\]
\end{proposition}

\begin{proof}
Applying the rules above to the constructions of \cref{def:strictterms,def:weakterms} gives the lengths in \cref{tab:lengths}, each obtained from those above it; for instance $\lh(\tilde\psi)=k-1+5k$, since $\lh(\xi(-z_i)\cdot z_i)=5$, and $\lh(\tilde\varepsilon_i)=9+\lh(\tilde\psi)+\lh(k)=8k+7$. For the strict construction, the right-hand side then has length $1+\lh(\tilde t_0)+1+(k-1)+\sum_i(3+\lh(\tilde t_i))=4k+1+(80k^2+160k+50)+k(72k^2+152k+60)$. For the weak construction, $\lh(\sigma(\tilde t_{-1})\cdot z_0)=3+\lh(\tilde t_{-1})$, and the right-hand side has length $4k+1+(749k+648)+k(707k+617)$.
\end{proof}

\begin{table}[t]
\centering\footnotesize
\caption{Lengths of the intermediate terms of the universal constructions, for $k\geq1$.}
\label{tab:lengths}
\begin{tabular}{@{}llll@{}}
\toprule
\multicolumn{2}{@{}l}{strict construction} & \multicolumn{2}{l}{weak construction}\\
\midrule
$\tilde\psi$ & $6k-1$ & $\tilde h$ & $7k-1$\\
$\tilde\varepsilon_i$ & $8k+7$ & $\tilde\varphi_1$; $\tilde\varphi_2,\tilde\varphi_3$ & $7k+10$; $7k+3$\\
$\tilde s_i$ & $8k+12$ & $\tilde\varphi$ & $21k+18$\\
$\tilde h$ & $8k^2+16k-1$ & $\tilde q$ & $35k+31$\\
$\tilde\varphi_1$; $\tilde\varphi_2,\tilde\varphi_3$ & $8k^2+16k+10$; $8k^2+16k+3$ & $\tilde\omega_1$; $\tilde\omega_2$; $\tilde\omega_3$ & $63k+55$; $77k+57$; $70k+66$\\
$\tilde\varphi$ & $24k^2+48k+18$ & $\tilde\omega$; $\tilde u$ & $210k+180$; $252k+215$\\
$\tilde w$ & $72k^2+144k+46$ & $\tilde\psi_1$; $\tilde\psi_2$; $\tilde\psi_3$ & $252k+216$; $210k+181$; $35k+32$\\
$\tilde u$ & $80k^2+160k+49$ & $\tilde\Psi_1$; $\tilde\Psi_2$; $\tilde\Psi_3$ & $504k+433$; $420k+363$; $70k+65$\\
$\tilde t_0$ & $80k^2+160k+50$ & $\tilde t_{-1}$; $\tilde t_0$ & $532k+464$; $749k+648$\\
$\tilde t_i$ & $72k^2+152k+60$ & $\tilde t_i$ & $707k+617$\\
\bottomrule
\end{tabular}
\end{table}

For $k=1$ the three lengths are $579$, $999$ and $2726$, and for $k=2$ they are $2003$, $1531$ and $6217$. These figures concern the particular constructions considered here; we do not claim that they are minimal. Expanded length counts every repeated subterm afresh, whereas identical subexpressions can be shared during evaluation. In the standard model of term graphs \cite{plump1999term}, a \emph{straight-line program} with inputs $z_0,\ldots,z_k$ is a finite directed acyclic graph whose nodes are labelled by constants or by function symbols applied to earlier nodes; its size is the number of nodes, its depth the length of a longest path, and each function symbol, inversion included, has unit cost.

\begin{proposition}\label{prop:slp}
For either construction, there is a straight-line program of size $O(k)$ and depth $O(\log(k+1))$ that computes all the terms $\tilde t_i$ simultaneously.
\end{proposition}

\begin{proof}
Compute each sum of $k$ terms, and the numeral $k$, by a balanced binary tree of additions, which has $k-1$ nodes and depth $\lceil\log_2k\rceil$. In the strict construction, the nodes $\xi(-z_i)\cdot z_i$ require $O(1)$ nodes each, $\tilde\psi$ and $k$ require $O(k)$ nodes and depth $O(\log(k+1))$, each $\tilde\varepsilon_i$ and $\tilde s_i$ then requires $O(1)$ further nodes, $\tilde h$ requires $O(k)$ nodes and depth $O(\log(k+1))$, the terms $\tilde\varphi_j,\tilde\varphi,\tilde w,\tilde u,\tilde t_0$ require $O(1)$ nodes in total, and each $\tilde t_i$ requires $O(1)$ further nodes. The weak construction is similar, with only the sum $\tilde h$ of length $k$. Every path passes through at most three balanced trees and $O(1)$ further nodes.
\end{proof}

If $g,f_1,\ldots,f_k$ are themselves computed by a straight-line program of size $S$ and depth $D$, the composed program has size $S+O(k)$ and depth $D+O(\log(k+1))$. These bounds concern the evaluation of the fixed universal terms; they say nothing about deciding the hypotheses of \cref{thm:strict,thm:weak}, which may be hard.

\section{Proofs of the Positivstellens\"{a}tze}\label{sec:proofs}

\subsection{Consequences of the scalar axioms}

\begin{lemma}\label{lem:sigmarho}
In every model of $T_{\wP}$ we have $\sigma(1)=1$, $\rho(0)=0$ and $\sigma(0)=0$, and $\sigma(a)>0$ and $\rho(a)>0$ for every $a>0$. Moreover, $\rho$ is injective and $\sigma$ is surjective as maps $[0,+\infty)\to[0,+\infty)$.
\end{lemma}

\begin{proof}
By \T{\rho}, \T{\sigma\rho} and \T{\sigma2}, $1=\sigma(\rho(1))=\sigma(\rho(1)\cdot1)=\sigma(\rho(1))\sigma(1)=\sigma(1)$. If $\rho(0)>0$, then $\sigma(1)=\sigma(\rho(0)\rho(0)^{-1})=\sigma(\rho(0))\sigma(\rho(0)^{-1})=0$, which is absurd; hence $\rho(0)=0$, and $\sigma(0)=\sigma(\rho(0))=0$. If $a>0$, then $1=\sigma(a)\sigma(a^{-1})$, so $\sigma(a)\neq0$ and $\sigma(a)>0$ by \T{\sigma1}; and $\rho(a)\geq0$ with $\sigma(\rho(a))=a\neq0=\sigma(0)$, so $\rho(a)>0$. Injectivity of $\rho$ and surjectivity of $\sigma$ on $[0,+\infty)$ follow from \T{\sigma\rho}.
\end{proof}

The axioms do not determine whether $\rho$ is surjective, or $\sigma$ injective or monotone, on $[0,+\infty)$; this is why continuity has to be assumed separately in \cref{prop:C0}.

\begin{example}[Pathological $\sigma$ and $\rho$]\label{ex:wild}
There are $\sigma,\rho:\R\to\R$ satisfying \T{\sigma1}--\T{\sigma\rho} such that, for every interval $(a,b)\subseteq[0,+\infty)$, neither $\sigma$ nor $\rho$ is monotone on $(a,b)$, $\sigma$ is not injective on $(a,b)$, and $(a,b)\not\subseteq\rho[(0,+\infty)]$.
\end{example}

\begin{proof}
Since $\R$ is a vector space over $\Q$ of infinite dimension, there is a $\Q$-linear isomorphism $H:\R\to\R\oplus\R$ (its existence uses a Hamel basis, and thus a fragment of the axiom of choice). Let $L:=\pi_1\circ H$ and $M:=H^{-1}\circ\iota_1$, where $\pi_1$ is the first projection and $\iota_1(t)=(t,0)$. Then $L$ and $M$ are additive, $L\circ M=\mathrm{id}_\R$, $L$ is surjective but not injective, and $M$ is injective but not surjective. A continuous additive map $\R\to\R$ is of the form $x\mapsto cx$ and is therefore either zero or bijective, so $L$ and $M$ are discontinuous; and an additive map that is monotone on a nonempty interval is continuous \cite[Theorem~1.1.7]{BGT_RV}, so neither $L$ nor $M$ is monotone on any interval. Put $\sigma(x):=\exp(L(\log x))$ and $\rho(x):=\exp(M(\log x))$ for $x>0$, and $\sigma(x)=\rho(x):=0$ for $x\leq0$. Then \T{\sigma1} and \T{\rho} are evident; \T{\sigma2} holds trivially if $ab=0$, and otherwise $\sigma(ab)=\exp(L(\log a)+L(\log b))=\sigma(a)\sigma(b)$; and $\sigma(\rho(a))=\exp(L(M(\log a)))=a$ for $a>0$. Since $\log$ and $\exp$ are increasing homeomorphisms, $\sigma$ and $\rho$ are monotone on no interval of $(0,+\infty)$. As $\ker L$ is a nonzero proper $\Q$-subspace, it is dense, so there are $u$ and $0\neq\kappa\in\ker L$ with $\log a<u<u+\kappa<\log b$, and then $\sigma(e^u)=\sigma(e^{u+\kappa})$. Finally, $\rho[(0,+\infty)]\subseteq\exp[M[\R]]$, and the proper $\Q$-subspace $M[\R]$ contains no interval.
\end{proof}

\subsection{Proof of the strict Positivstellensatz}\label{sec:proofstrict}

The proof follows Fischer's construction \cite[Lemmas~2.1--2.2 and Theorem~1.1]{FischerPsatzDiff} through the functions $(\psi,\varepsilon_i,s_i)\leadsto h\leadsto(\varphi_j,w)\leadsto(u,t_i)$ defined by the terms of \cref{def:strictterms}.

\begin{lemma}[cf.\ {\cite[Lemma~2.1]{FischerPsatzDiff}}]\label{lem:eps}
Let $\bR\models T_{\sP}$, let $A$ satisfy \textup{(A0)}, \textup{(A1)}, \textup{(A$\delta$)} and \textup{(A$\nu$)}, and let $g,\varphi,f\in A$ satisfy $\{g\leq0\}\subseteq\{\varphi>0\}$. Then $\varepsilon:=\nu(g,\varphi)/\delta(f)$ lies in $A$, is positive on $X$, and satisfies $\varepsilon f<\varphi$ on $\{g\leq0\}$.
\end{lemma}

\begin{proof}
By (A$\delta$) and \T{\delta}, $\delta(f)\in A$ is positive, so $\delta(f)^{-1}\in A$ by (A1); and $\nu(g,\varphi)\in A$ by (A$\nu$). Hence $\varepsilon\in A$. At each point either $g>0$ or $\varphi>0$, so $\nu(g,\varphi)>0$ by \T{\nu1}, and $\varepsilon>0$. Let $g(x)\leq0$. Then $\varphi(x)>0$, and $0<\nu(g(x),\varphi(x))\leq\varphi(x)$ by \T{\nu1} and \T{\nu2}. If $f(x)\leq0$, then $\varepsilon(x)f(x)\leq0<\varphi(x)$. If $f(x)>0$, then $0<f(x)/\delta(f(x))\leq\frac12$ by \T{\delta}, so $\varepsilon(x)f(x)\leq\varphi(x)/2<\varphi(x)$.
\end{proof}

\emph{For the rest of this subsection we fix $k\geq0$, $\bR\models T_{\sP}$, $A\subseteq R^X$ satisfying \textup{(AsP)}, and $g,f_1,\ldots,f_k\in A$ with $F:=\bigcap_i\{f_i\geq0\}\subseteq\{g>0\}$.} We denote by $\psi,\varepsilon_i,s_i,h,\varphi_j,\varphi,w,u,t_i$ the functions obtained by substituting $g,f_1,\ldots,f_k$ into the terms of \cref{def:strictterms}.

\begin{lemma}\label{lem:psi}
$\psi\in A$, $\psi\leq0$ on $X$, and $\{g\leq0\}\subseteq\{\psi<0\}$. Moreover, for $1\leq i\leq k$: $\varepsilon_i\in A$, $\varepsilon_i>0$ on $X$, and $\varepsilon_if_i<-\psi/k$ on $\{g\leq0\}$; and $s_i\in A$, $s_i>0$ on $X$, and $\sigma(s_i)=\xi(-f_i)+\varepsilon_i$.
\end{lemma}

\begin{proof}
By (A2) and (A0), $\psi\in A$. By \T{\xi1} and \T{\xi2}, $\xi(-f_i(x))=0$ if $f_i(x)\geq0$ and $\xi(-f_i(x))>0$ if $f_i(x)<0$; hence every summand $\xi(-f_i)f_i$ is nonpositive, and negative exactly where $f_i<0$. If $g(x)\leq0$, then $x\notin F$, so $f_i(x)<0$ for some $i$ and $\psi(x)<0$. Let $k\geq1$. The constant $k$ lies in $A$ and is positive, so $-\psi/k\in A$ by (A0) and (A1), and $\{g\leq0\}\subseteq\{-\psi/k>0\}$; \cref{lem:eps} with $\varphi=-\psi/k$ and $f=f_i$ gives the assertions on $\varepsilon_i$. Finally, $\xi(-f_i)+\varepsilon_i\in A$ is positive, so $s_i\in A$ by (A3s), $s_i>0$ by \cref{lem:sigmarho}, and $\sigma(s_i)=\xi(-f_i)+\varepsilon_i$ by \T{\sigma\rho}.
\end{proof}

\begin{lemma}[cf.\ {\cite[Lemma~2.2]{FischerPsatzDiff}}]\label{lem:hstrict}
$h\in A$, $h=\psi+\sum_i\varepsilon_if_i$, and $F\subseteq\{h\geq0\}\subseteq\{g>0\}$.
\end{lemma}

\begin{proof}
If $k=0$, then $h=0$ and $F=X=\{g>0\}$. Let $k\geq1$. By \cref{lem:psi} and (A3s), $\sigma(s_i)\in A$, so $h\in A$, and the identity follows from $\sigma(s_i)=\xi(-f_i)+\varepsilon_i$. On $F$ every summand $\sigma(s_i)f_i$ is nonnegative by \T{\sigma1}, so $h\geq0$. If $g(x)\leq0$, then summing the inequalities $\varepsilon_i(x)f_i(x)<-\psi(x)/k$ gives $h(x)<\psi(x)-\psi(x)=0$.
\end{proof}

\begin{lemma}\label{lem:w}
$\varphi_1,\varphi_2,\varphi_3,\varphi\in A$ and $\varphi>0$ on $X$. Moreover, $w\in A$, $w>0$ on $X$, $w=g/(g+h)$ on $\{h\geq0\}$, and $w=(g+h)/h$ on $\{g\leq0\}$.
\end{lemma}

\begin{proof}
By (A0) and (A2), $\varphi_j,\varphi\in A$. Put $U_1:=\{g+4h>0\}$, $U_2:=\{g+h<0\}$ and $U_3:=\{g>0\}\cap\{h<0\}$. By \T{\xi2}, $\{\varphi_j>0\}=U_j$ for $j=1,2$; and $\{\varphi_3>0\}=\{gh<0\}=U_3$, because $\{g<0\}\cap\{h>0\}=\varnothing$ by \cref{lem:hstrict}. The sets $U_j$ cover $X$: if $h(x)\geq0$, then $g(x)>0$ and $x\in U_1$; if $h(x)<0$, then $x\in U_2$ when $g(x)\leq0$ and $x\in U_3$ when $g(x)>0$. Hence $\varphi>0$ on $X$ and $\varphi^{-1}\in A$ by (A1).

Next, $\{g+4h\geq0\}\cap\{g+h=0\}=\varnothing$: if $g+h=0$ and $g+4h\geq0$ at $x$, then $-3g(x)\geq0$, so $g(x)\leq0$ and $h(x)=-g(x)\geq0$, contrary to \cref{lem:hstrict}. Also $\{-(g+h)\geq0\}\cap\{h=0\}=\varnothing$, since $h(x)=0$ forces $g(x)>0$. By (A4s), therefore, $\varphi_1(g+h)^{-1}\in A$ and $\varphi_2h^{-1}\in A$, and $w\in A$ by (A0).

If $h(x)\geq0$, then $g(x)>0$, $x\in U_1\setminus(U_2\cup U_3)$, and $w(x)=g(x)/(g(x)+h(x))>0$. If $g(x)\leq0$, then $h(x)<0$, $x\in U_2\setminus U_3$, and $x\notin U_1$, since $g+4h>0$ and $h<0$ would force $g>0$; thus $w(x)=(g(x)+h(x))/h(x)>0$. If $g(x)>0>h(x)$, then $\varphi_3(x)>0$; moreover, $\varphi_1(x)>0$ implies $g(x)+h(x)>g(x)+4h(x)>0$, and $\varphi_2(x)>0$ implies $g(x)+h(x)<0$, so that the first two summands in the definition of $w(x)$ are nonnegative and the third is positive. Hence $w(x)>0$.
\end{proof}

\begin{lemma}\label{lem:ustrict}
$u\in A$ and $u>0$ on $X$; $t_0,\ldots,t_k\in A$ are positive on $X$; and $\sigma(t_i)=w\,\sigma(s_i)$ for $1\leq i\leq k$.
\end{lemma}

\begin{proof}
By (A0), $u=g-wh\in A$. By \cref{lem:w}: if $h(x)\geq0$, then $u(x)=g(x)-g(x)h(x)/(g(x)+h(x))=g(x)^2/(g(x)+h(x))>0$; if $g(x)\leq0$, then $u(x)=g(x)-(g(x)+h(x))=-h(x)>0$; and if $g(x)>0>h(x)$, then $u(x)=g(x)+w(x)(-h(x))>0$. By (A3s) and \cref{lem:sigmarho}, $t_0=\rho(u)$ and $\rho(w)$ lie in $A$ and are positive, and so is $t_i=\rho(w)s_i$ by \cref{lem:psi}. Finally, $\sigma(t_i)=\sigma(\rho(w))\sigma(s_i)=w\sigma(s_i)$ by \T{\sigma2} and \T{\sigma\rho}.
\end{proof}

\begin{proof}[Proof of \cref{thm:strict}]
By \cref{lem:ustrict}, the $t_i$ lie in $A$ and are positive, and, using \T{\sigma\rho},
\[
\sigma(t_0)+\textstyle\sum_i\sigma(t_i)f_i\ =\ u+w\sum_i\sigma(s_i)f_i\ =\ u+wh\ =\ g .\qedhere
\]
\end{proof}

\subsection{Proof of the weak Positivstellensatz}\label{sec:proofweak}

The proof follows \cite[Lemma~2.3 and Theorem~1.2]{FischerPsatzDiff} through the functions $h\leadsto\varphi\leadsto q\leadsto\omega\leadsto u\leadsto(\psi_j,\Psi_j)\leadsto(p,t_i)$ of \cref{def:weakterms}. We begin with a version of Hilbert's seventeenth problem in terms of $\sigma$.

\begin{lemma}[cf.\ {\cite[Lemma~2.3]{FischerPsatzDiff}}]\label{lem:H17}
Let $\bR\models T_{\wP}$, let $A$ satisfy \textup{(A0)}, \textup{(A2)} and \textup{(A3w)}, and let $f\in A$ be nonnegative on $X$. Put $\psi:=\xi(f)$ and $\Psi:=\rho(f)\xi(f)$. Then $\psi,\Psi,\sigma(\psi),\sigma(\Psi)$ are nonnegative functions in $A$, $\sigma(\psi)f=\sigma(\Psi)$, and $\{\psi=0\}=\{\Psi=0\}=\{f=0\}$.
\end{lemma}

\begin{proof}
We have $\psi\in A$ by (A2), and $\Psi,\sigma(\psi),\sigma(\Psi)\in A$ by (A3w); nonnegativity follows from \T{\xi1}, \T{\rho} and \T{\sigma1}. Pointwise, \T{\sigma2} and \T{\sigma\rho} give $\sigma(\Psi)=\sigma(\rho(f))\sigma(\psi)=f\sigma(\psi)$. Since $f\geq0$, \T{\xi2} gives $\psi(x)=0$ if and only if $f(x)=0$; and if $f(x)=0$ then $\Psi(x)=0$, whereas if $f(x)>0$ then $\Psi(x)>0$ by \T{\xi2} and \cref{lem:sigmarho}.
\end{proof}

\emph{For the rest of this subsection we fix $k\geq0$, $\bR\models T_{\wP}$, $A\subseteq R^X$ satisfying \textup{(AwP)}, and $g,f_1,\ldots,f_k\in A$ with $F:=\bigcap_i\{f_i\geq0\}\subseteq\{g\geq0\}$,} and we denote by $h,\varphi_j,\varphi,q,\omega_j,\omega,u,\psi_j,\Psi_j,p,t_i$ the functions obtained from the terms of \cref{def:weakterms}.

\begin{lemma}\label{lem:hweak}
$h\in A$, $h\leq0$ on $X$, and $F=\{h=0\}=\{h\geq0\}\subseteq\{g\geq0\}$.
\end{lemma}

\begin{proof}
By (A2) and (A3w), $\sigma(\xi(-f_i))\in A$, so $h\in A$. By \T{\xi2}, \T{\sigma1} and \cref{lem:sigmarho}, the summand $\sigma(\xi(-f_i))f_i$ vanishes where $f_i\geq0$ and is negative where $f_i<0$. Hence $h\leq0$, and $h(x)=0$ if and only if $x\in F$.
\end{proof}

\begin{lemma}\label{lem:phiweak}
$\varphi_1,\varphi_2,\varphi_3,\varphi$ are nonnegative functions in $A$, and $\{\varphi=0\}=\{g=0\}\cap\{h=0\}$.
\end{lemma}

\begin{proof}
Membership and nonnegativity follow from (A0), (A2) and \T{\xi1}. By \T{\xi2}, $\varphi(x)=0$ if and only if $g+4h\leq0$, $g+h\geq0$ and $gh\geq0$ at $x$. If moreover $h(x)<0$, then $g(x)\geq-h(x)>0$ and $g(x)h(x)<0$, which is impossible; so $h(x)=0$, whence $g(x)\geq0$ by \cref{lem:hweak} and $g(x)\leq0$ from $g+4h\leq0$. Conversely, $g(x)=h(x)=0$ gives $\varphi(x)=0$.
\end{proof}

\begin{lemma}\label{lem:q}
$q$ is a nonnegative function in $A$, and $\{q=0\}=\{g=0\}$.
\end{lemma}

\begin{proof}
Membership and nonnegativity are clear. If $g(x)=0$, then $\xi(g(x)^2)=0$. If $g(x)\neq0$, then $\xi(g(x)^2)>0$ and $\varphi(x)>0$ by \cref{lem:phiweak}; moreover, if $h(x)<0$ then $\xi(-h(x))>0$, and if $h(x)=0$ then $g(x)>0$ by \cref{lem:hweak}, so that $\xi(g(x))\xi(g(x)+h(x))>0$. Hence $q(x)>0$.
\end{proof}

\begin{lemma}\label{lem:omega}
$\omega$ is a nonnegative function in $A$, $\{\omega=0\}=\{g=0\}$, and $q(x)g(x)-\omega(x)h(x)>0$ whenever $g(x)\neq0$.
\end{lemma}

\begin{proof}
Put $c:=\xi(g^2)(\xi(-h)+\xi(g)\xi(g+h))\in A$, so that $q=\varphi c$. Off $\{\varphi=0\}$ the factor $\varphi$ cancels, and on $\{\varphi=0\}=\{g=0\}\cap\{h=0\}$ both sides below vanish; hence
\[
\omega_1=\varphi_3c,\qquad \omega_2=c\varphi_2(g+h)/h,\qquad \omega_3=c\varphi_1g/(g+h).
\]
Clearly $\omega_1\in A$. Where $\varphi_2>0$ we have $g+h<0$ and hence $\xi(g+h)=0$; thus $\varphi_2\xi(g)\xi(g+h)=0$, and $\omega_2=\xi(g^2)\varphi_2(g+h)\cdot\xi(-h)h^{-1}\in A$ by (A4w). Since $\{h\geq0\}\subseteq\{g\geq0\}$ by \cref{lem:hweak}, (A5w) gives $\xi(g^2)\varphi_1(g+h)^{-1}\in A$, and $\omega_3=(\xi(-h)+\xi(g)\xi(g+h))\,g\cdot\xi(g^2)\varphi_1(g+h)^{-1}\in A$. Thus $\omega\in A$.

Each $\omega_j$ is nonnegative: this is clear for $\omega_1$; if $\varphi_2(x)>0$, then $g(x)+h(x)<0$, hence $h(x)<0$ (as $h(x)=0$ would force $g(x)\geq0$) and $(g+h)/h>0$ at $x$; and if $\varphi_1(x)>0$, then $g(x)+4h(x)>0$ with $h(x)\leq0$, so that $g(x)>0$ and $g(x)+h(x)>0$ (otherwise $h(x)<0$ and $g+4h=(g+h)+3h<0$), whence $g/(g+h)>0$ at $x$.

If $g(x)=0$, then $q(x)=0$ by \cref{lem:q} and every $\omega_j(x)=0$. Let $g(x)\neq0$, so that $q(x)>0$ and $\varphi(x)>0$. If $g(x)<0$, then $h(x)<0$ and $\varphi_2(x)>0$, so $\omega_2(x)>0$; if $g(x)>0=h(x)$, then $\varphi_1(x)>0$ and $\omega_3(x)=q(x)\varphi_1(x)/\varphi(x)>0$; and if $g(x)>0>h(x)$, then $\varphi_3(x)>0$ and $\omega_1(x)>0$. In every case $\omega(x)>0$. Finally, if $g(x)>0$, then $qg-\omega h\geq qg>0$ at $x$, as $h\leq0$. If $g(x)<0$, then $h(x)<0$, $\varphi_2(x)>0$ and $\varphi_1(x)=\varphi_3(x)=0$, so that $\omega(x)=\omega_2(x)=q(x)(g(x)+h(x))/h(x)$ and $q(x)g(x)-\omega(x)h(x)=-q(x)h(x)>0$.
\end{proof}

\begin{proof}[Proof of \cref{thm:weak}]
By \cref{lem:omega}, $u=qg-\omega h$ lies in $A$, is nonnegative, and $\{u=0\}=\{g=0\}$. \Cref{lem:H17}, applied to $u$, $\omega$ and $q$, shows that $\psi_j,\Psi_j$ are nonnegative functions in $A$ with
\[
\sigma(\psi_1)u=\sigma(\Psi_1),\quad \sigma(\psi_2)\omega=\sigma(\Psi_2),\quad \sigma(\psi_3)q=\sigma(\Psi_3),
\]
and $\{\psi_1=0\}=\{\Psi_1=0\}=\{u=0\}$, and similarly for $\omega$ and $q$. Hence $p,t_0,\ldots,t_k$ are nonnegative functions in $A$, by (A0), (A2) and \T{\xi1}, and $\{p=0\}=\{u=0\}\cup\{\omega=0\}\cup\{q=0\}=\{g=0\}$ by \cref{lem:q,lem:omega}. Using \T{\sigma2} repeatedly, as all factors are nonnegative,
\begin{align*}
\sigma(t_0)+\textstyle\sum_i\sigma(t_i)f_i
&=\sigma(\Psi_1)\sigma(\psi_2)\sigma(\psi_3)+\sigma(\psi_1)\sigma(\Psi_2)\sigma(\psi_3)\textstyle\sum_i\sigma(\xi(-f_i))f_i\\
&=\sigma(\psi_1)\sigma(\psi_2)\sigma(\psi_3)\,(u+\omega h)=\sigma(\psi_1)\sigma(\psi_2)\sigma(\psi_3)\,qg\\
&=\sigma(\psi_1)\sigma(\psi_2)\sigma(\Psi_3)\,g=\sigma(p)\,g .\qedhere
\end{align*}
\end{proof}

\begin{remark}[The use of the axioms]\label{rem:axiomuse}
In the strict construction, (A1) is used for the divisions in $\varepsilon_i$ and $w$, (A2) for $\psi$ and the gates $\varphi_j$, (A3s) for $s_i$ and $t_i$, (A4s) for $w$, and (A$\delta$) and (A$\nu$) for $\varepsilon_i$. In the weak construction, (A2) and (A3w) are used for $h$ and in \cref{lem:H17}, and (A4w) and (A5w) are used only to show $\omega_2,\omega_3\in A$; (A1) is not used. Consequently \cref{thm:weak} remains true if (A4w) and (A5w) are replaced by the \emph{a priori} weaker axioms
\begin{enumerate}
\item[(A4w$^-$)] if $g,h\in A$, then $\xi(g^2)\xi(-h)\xi(-(g+h))(g+h)h^{-1}\in A$;
\item[(A5w$^-$)] if $g,h\in A$ and $\{h\geq0\}\subseteq\{g\geq0\}$, then $\xi(g^2)(\xi(-h)+\xi(g)\xi(g+h))\xi(g+4h)\,g\,(g+h)^{-1}\in A$.
\end{enumerate}
\end{remark}

\begin{lemma}\label{lem:A4wminus}
If $B\subseteq R^X$ satisfies \textup{(A0)}, \textup{(A1)} and \textup{(A2)}, then $B$ satisfies \textup{(A4w)} if and only if it satisfies \textup{(A4w$^-$)}.
\end{lemma}

\begin{proof}
If (A4w) holds, then (A4w$^-$) follows by (A0) and (A2). Conversely, let $f\in B$, and apply (A4w$^-$) to $(g,h)=(-(1+f^2),f)$. Then $-(g+h)=:P=f^2-f+1=(f-\frac12)^2+\frac34>0$ and $g^2>0$, so $\beta:=\xi(g^2)\xi(P)P\in B$ is positive and $\beta^{-1}\in B$ by (A1). Since $\xi(g^2)\xi(-f)\xi(P)(-P)f^{-1}=-\beta\,\xi(-f)f^{-1}\in B$, we obtain $\xi(-f)f^{-1}\in B$.
\end{proof}

Thus, in the presence of (A1), nothing is gained by (A4w$^-$); without (A1) we do not know whether (A4w$^-$) implies (A4w). Nor do we know whether there is an algebra satisfying (AwP) with (A5w) replaced by (A5w$^-$), but not (A5w) itself.

\subsection{Independence of (A1) from the weak axioms}\label{sec:A1indep}

\begin{proposition}\label{prop:A1indep}
There are a model $\bR\models T_{\wP}$, a set $X$ and a set $A\subseteq R^X$ satisfying \textup{(AwP)} but not \textup{(A1)}.
\end{proposition}

\begin{proof}
Let $R$ be a non-Archimedean ordered field, and let $\mathcal O:=\{a\in R:|a|\leq n\text{ for some }n\in\N\}$ be its ring of finite elements. Then $\mathcal O$ is a subring of $R$ containing $\Q$, and it is convex: if $0\leq b\leq a$ and $a\in\mathcal O$, then $b\in\mathcal O$. Interpret $\sigma$, $\rho$ and $\xi$ alike by $a\mapsto a^+:=\max\{a,0\}$; then $\bR\models T_{\wP}$, as in \cref{lem:relu}. Let $X:=\{\ast\}$ and identify $A:=\mathcal O$ with the ring of constant functions on $X$ with values in $\mathcal O$. Then (A0) holds, and (A2) holds because $0\leq a^+\leq|a|$ and $\mathcal O$ is convex. If $a\in\mathcal O$ is nonnegative, then $\sigma(a)=a$ and $\rho(a)\xi(a)=a^2$ lie in $\mathcal O$, so (A3w) holds. For (A4w), $\xi(-a)a^{-1}$ equals $-1$ if $a<0$ and $0$ otherwise. For (A5w), let $g,h\in\mathcal O$ with $h\geq0\Rightarrow g\geq0$, and put $E:=\xi(g^2)\xi(g+4h)(g+h)^{-1}=g^2(g+4h)^+(g+h)^{-1}$. If $g+4h\leq0$ or $g=0$, then $E=0$. Otherwise $g+4h>0$, so $g>0$ (by hypothesis if $h\geq0$, and since $g>-4h$ if $h<0$), and $g+h\geq\frac34g>0$, as in the proof of \cref{lem:fischer-regularity}; hence $0\leq E=g^2(g+4h)/(g+h)\leq\frac43g(g+4h)\in\mathcal O$, and $E\in\mathcal O$ by convexity. Thus $A$ satisfies (AwP). Finally, choose $H\in R$ with $H>n$ for every $n\in\N$. Then $0<H^{-1}<1$, so $H^{-1}\in\mathcal O$ is positive, but $(H^{-1})^{-1}=H\notin\mathcal O$, and (A1) fails.
\end{proof}

\section{Verification of the examples}\label{sec:verification}

\subsection{Definable functions}

\begin{proof}[Proof of \cref{cor:definable}]
Each axiom asserts that a certain composite of functions in $A$ with the field operations and the auxiliary functions lies in $A$. As compositions of definable functions are definable, every such composite is definable, hence in $A$; for (A4s), for example, $\xi(f)g^{-1}$ is definable for all $f,g\in A$, whether or not $\{f\geq0\}\subseteq\{g\neq0\}$.
\end{proof}

\subsection{Continuous functions}\label{sec:proofC0}

Let $\bR\models T_{\oF}$ carry the order topology, let $X$ be a topological space, and $A=C^0(X,R)$. The field operations are continuous, inversion being continuous on $R\setminus\{0\}$. We use repeatedly the elementary fact that if $Y\subseteq R^m$, $H:Y\to R$ is continuous, and $f_1,\ldots,f_m\in A$ satisfy $(f_1(x),\ldots,f_m(x))\in Y$ for all $x\in X$, then $H(f_1,\ldots,f_m)\in A$. Throughout, $\xi:R\to R$ is continuous and satisfies \T{\xi1} and \T{\xi2}, so that $\xi=0$ on $(-\infty,0]$.

\begin{lemma}[For (A4s)]\label{lem:Psi}
On $S:=\{(x,y)\in R^2:x\geq0\Rightarrow y\neq0\}$, the function $\Psi(x,y):=\xi(x)/y$, with $\Psi(x,0):=0$, is continuous.
\end{lemma}

\begin{proof}
Near a point with $y_0\neq0$, $\Psi$ is a composite of continuous functions. If $y_0=0$, then $x_0<0$, and $\Psi=0$ on the neighbourhood $\{x<0\}$ of $(x_0,y_0)$.
\end{proof}

\begin{lemma}[For (A4w)]\label{lem:Theta}
The function $\Theta(x):=\xi(-x)/x$, with $\Theta(0):=0$, is continuous on $R\setminus\{0\}$, and it is continuous on $R$ if and only if $\xi(t)\prec t$ at $0^+$.
\end{lemma}

\begin{proof}
The first assertion is clear, and $\Theta=0$ on $(0,+\infty)$. For $t>0$ we have $|\Theta(-t)|=\xi(t)/t$, so continuity at $0$ amounts to $\xi(t)/t\to0$ as $t\downarrow0$, which is the condition $\xi(t)\prec t$ at $0^+$.
\end{proof}

\begin{lemma}[For (A5w)]\label{lem:Phi}
On $Q:=\{(x,y)\in R^2:y\geq0\Rightarrow x\geq0\}$, the function $\Phi(x,y):=\xi(x+4y)/(x+y)$, with $\Phi(x,y):=0$ if $x+y=0$, is continuous on $Q\setminus\{(0,0)\}$, and it is continuous on $Q$ if and only if $\xi(t)\prec t$ at $0^+$.
\end{lemma}

\begin{proof}
Near a point of $Q$ with $x_0+y_0\neq0$, $\Phi$ is a composite of continuous functions. If $x_0+y_0=0$ and $(x_0,y_0)\neq(0,0)$, then $y_0<0<x_0$ by the definition of $Q$, so $x_0+4y_0=-3x_0<0$, and $\Phi=0$ near $(x_0,y_0)$. Since $(t,0)\in Q$ and $\Phi(t,0)=\xi(t)/t$ for $t>0$, continuity at $(0,0)$ implies $\xi(t)\prec t$ at $0^+$. Conversely, let $\varepsilon>0$, choose $\eta_0>0$ with $\xi(t)<\frac\varepsilon4t$ for $0<t<\eta_0$, and let $(x,y)\in Q$ with $|x|,|y|<\eta_0/5$ and $\Phi(x,y)\neq0$. Then $x+4y>0$, and $x+y>0$: if $y\geq0$, then $x\geq0$ and $x+y\neq0$; if $y<0$ and $x+y\leq0$, then $x+4y=(x+y)+3y<0$. Moreover $x+4y\leq4(x+y)$, since $x\geq0$ when $y\geq0$, and $x+4y<x+y$ when $y<0$. As $0<x+4y<\eta_0$, we get $0\leq\Phi(x,y)<\frac\varepsilon4(x+4y)/(x+y)\leq\varepsilon$.
\end{proof}

\begin{lemma}\label{lem:rhocont}
Let $\bR\models T_{\wP}$, and suppose that $\xi$ is continuous, that $\sigma$ restricted to $[0,+\infty)$ is continuous at $0$, and that $t\mapsto\xi(t)\rho(t)$ is continuous on $[0,+\infty)$. Then $\rho$ is continuous on $[0,+\infty)$.
\end{lemma}

\begin{proof}
We first show that for every $\varepsilon>0$ there is $\eta>0$ with $\sigma(y)\geq\eta$ for all $y\geq\varepsilon$. Choose $\eta_0>0$ with $\sigma(z)<1$ for $0\leq z<\eta_0$, and put $c:=\varepsilon\eta_0/2$ and $\eta:=\sigma(c)>0$ (\cref{lem:sigmarho}). If $y\geq\varepsilon$ and $\sigma(y)<\eta$, then $z:=c/y\in(0,\eta_0)$, while by \T{\sigma2} and \cref{lem:sigmarho}, $\sigma(z)=\sigma(c)\sigma(y^{-1})=\sigma(c)/\sigma(y)>1$, a contradiction. Now, if $0\leq t<\eta$ and $\rho(t)\geq\varepsilon$, then $t=\sigma(\rho(t))\geq\eta$, which is absurd; hence $0\leq\rho(t)<\varepsilon$ for $0\leq t<\eta$, and $\rho$ is continuous at $0=\rho(0)$. At $a>0$ we have $\xi\neq0$ near $a$, and $\rho=(\xi\rho)/\xi$ is continuous there.
\end{proof}

\begin{proof}[Proof of \cref{prop:C0}]
(1) (A0) and (A1) hold because the field operations are continuous, (A2) and (A$\delta$) because $\xi$ and $\delta$ are continuous, and (A3s) because $\sigma$ and $\rho$ are continuous on $(0,+\infty)$. If $\{f\geq0\}\subseteq\{g\neq0\}$, then $(f,g)$ takes values in $S$, and $\xi(f)g^{-1}=\Psi(f,g)\in A$ by \cref{lem:Psi}. If $\{f\leq0\}\subseteq\{g>0\}$, then $(f,g)$ takes values in the domain of continuity of $\nu$, so $\nu(f,g)\in A$.
(2) (A0) and (A2) hold as before ((A1) holds as well, although it is not part of (AwP)), and (A3w) holds by the continuity of $\sigma$ and $\xi\rho$ on $[0,+\infty)$. By \cref{lem:Theta}, $\Theta$ is continuous, so $\xi(-f)f^{-1}=\Theta(f)\in A$. If $\{h\geq0\}\subseteq\{g\geq0\}$, then $(g,h)$ takes values in $Q$, and $\xi(g+4h)(g+h)^{-1}=\Phi(g,h)\in A$ by \cref{lem:Phi}; multiplying by $\xi(g^2)\in A$ gives (A5w). The last assertion is \cref{lem:rhocont}.
\end{proof}

\subsection{Particular auxiliary functions}

\begin{proof}[Proof of \cref{lem:relu}]
$\ReLU(x)=(x+|x|)/2$ is continuous, hence so are $\xi$, $\sigma$, $\rho$, $\delta$, and $\nu=\max$, as $\max(x,y)=y+\ReLU(x-y)$. \T{\xi1} and \T{\xi2} hold since $\ReLU(a)^s>0$ if and only if $a>0$. \T{\sigma1}, \T{\rho} are clear; for $a,b\geq0$ we have $\ReLU(ab)=ab=\ReLU(a)\ReLU(b)$ and $\ReLU(\ReLU(a))=a$. \T{\delta}: $2\ReLU(a)+1>0$, and $2\ReLU(a)+1\geq2a$. \T{\nu1}: $\max(a,b)>0$ if $a>0$ or $b>0$. \T{\nu2}: if $a\leq0<b$, then $\max(a,b)=b$. Finally, if $s\geq2$ and $0<t<\min(1,\varepsilon)$, then $\xi(t)=t^{s-1}t\leq t\cdot t<\varepsilon t$.
\end{proof}

\begin{proof}[Proof of \cref{lem:sqrt}]
The map $x\mapsto x^2$ is a strictly increasing bijection of $[0,+\infty)$, so $\sqrt{\cdot}$ is strictly increasing; it is continuous at $0$, since $\sqrt x<\varepsilon$ for $x<\varepsilon^2$, and at $a>0$, since $|\sqrt x-\sqrt a|=|x-a|/(\sqrt x+\sqrt a)\leq|x-a|/\sqrt a$. Hence $\sigma,\rho,\delta,\nu$ are continuous. \T{\sigma1}--\T{\sigma\rho} are immediate. \T{\delta}: $2a\leq2|a|=2\sqrt{a^2}\leq2\sqrt{1+a^2}$, and $\delta>0$. \T{\nu1}: if $a>0$, then $\nu(a,b)\geq a>0$; if $b>0$, then $\sqrt{a^2+b^2}>|a|\geq-a$. \T{\nu2}: if $a\leq0<b$, then $a^2+b^2\leq(b-a)^2$ since $ab\leq0$, so $\nu(a,b)\leq(b-a)+a=b$.
\end{proof}

\begin{proof}[Proof of \cref{cor:mainexample,cor:rational}]
In the main example, the axioms on $\xi$ hold by \cref{lem:relu} with $s=2$, and the remaining axioms of $T_{\sP}$ by \cref{lem:sqrt}. The conditions of \cref{prop:C0} hold by the same two lemmas, since $\xi(t)\rho(t)=t^{5/2}$ for $t\geq0$. For $\bm Q$, the axioms and the conditions of \cref{prop:C0} hold by \cref{lem:relu} with $s=2$, since $\xi(t)\rho(t)=t^3$ for $t\geq0$.
\end{proof}

\subsection{Fischer's setting}\label{sec:prooffischer}

In the setting of \cref{sec:fischer}, expand $\bR$ by $\sigma(x):=x^{2d}$, $\rho(x):=|x|^{1/2d}$, $\delta(x):=2\sqrt{1+x^2}$, $\nu(x,y):=\sqrt{x^2+y^2}+x$, and $\xi:=\ReLU^{2r+2}$ if $r<\infty$, while for $r=\infty$ we put $\xi(x):=\exp(-1/x)$ for $x>0$ and $\xi(x):=0$ otherwise. These functions are definable in $\bR$. Let $X:=M$ and let $A$ be the set of definable $C^r$-functions $M\to R$.

\begin{lemma}\label{lem:fischer-scalar}
The expanded structure is a model of $T_{\sP}$.
\end{lemma}

\begin{proof}
For $r<\infty$, \T{\xi1}--\T{\xi2} hold by \cref{lem:relu}; for $r=\infty$ they hold because $\exp>0$. \T{\delta} and \T{\nu1}--\T{\nu2} hold by \cref{lem:sqrt}, $R$ being real closed. Finally, $x^{2d}\geq0$, $(ab)^{2d}=a^{2d}b^{2d}$, and $(a^{1/2d})^{2d}=a$ for $a\geq0$.
\end{proof}

\begin{lemma}\label{lem:fischer-regularity}
The functions $U(t):=\rho(t)\xi(t)$ and $V(t):=\xi(-t)t^{-1}$, with $0^{-1}=0$, are definable and $C^r$ on $R$. If $g,h\in A$ satisfy $\{h\geq0\}\subseteq\{g\geq0\}$, then $W:=\xi(g^2)\xi(g+4h)(g+h)^{-1}$ lies in $A$.
\end{lemma}

\begin{proof}
Definability is clear. Let $r<\infty$. Then $U(t)=t^{2r+2+1/2d}$ for $t>0$ and $U(t)=0$ otherwise, and every derivative of order at most $r$ tends to $0$ as $t\downarrow0$; and $V(t)=-(-t)^{2r+1}$ for $t<0$ and $V(t)=0$ otherwise. Both are therefore $C^r$. For $W$, work in a chart and put $\mathcal U:=\{g+4h>0\}$. Away from $\{g+h=0\}$ the claim follows from the rules of calculus, and $W$ vanishes near every point of $\{g+h=0\}$ outside the closure of $\mathcal U$. At a point of $\{g+h=0\}$ in the closure of $\mathcal U$, the hypothesis forces $g=h=0$ (as in the proof of \cref{lem:Phi}). On $\mathcal U$ we have $g+h>0$ and $g+h\geq\frac34g\geq0$: if $h\geq0$, then $g\geq0$ by hypothesis and $g+h\geq g$; if $h<0$, then $g>-4h>0$ and $g+h>\frac34g$. On $\mathcal U$, moreover, $\xi(g^2)=g^{4r+4}$, and an induction using the product and reciprocal rules gives $D^\ell\bigl(g^{4r+4}/(g+h)\bigr)=g^{4r+4-\ell}(g+h)^{-\ell-1}\tau_\ell$ with $\tau_\ell$ continuous, and bounded after shrinking the chart. This derivative vanishes where $g=0$, and is bounded by $C_\ell\,g^{4r+3-2\ell}$ where $g>0$; for $\ell\leq r$ it therefore tends to $0$ at the point in question. By the Leibniz rule, multiplication by the $C^r$ function $\xi(g+4h)$ preserves the vanishing of the derivatives of order at most $r$, and the extension of $W$ by $0$ is $C^r$; this is the estimate used in \cite[proof of Theorem~1.2]{FischerPsatzDiff}. For $r=\infty$, the same arguments apply: derivatives of $U(t)$ and $V(-t)$ for $t>0$ are bounded near $0$ by $Ct^{-N}\exp(-1/t)$, while derivatives of $\xi(g^2)/(g+h)$ on $\mathcal{U}\cap\{g>0\}$ are bounded near $g=0$ by $Cg^{-N}\exp(-1/g^2)$. Both bounds tend to $0$, since $\exp(s)\geq s^m/m!$ for $s>0$ and every $m$; where $g=0$, all derivatives of $\xi(g^2)$ vanish.
\end{proof}

\begin{lemma}\label{lem:fischer-algebra}
The algebra $A$ satisfies \textup{(AsP)} and \textup{(AwP)}.
\end{lemma}

\begin{proof}
(A0), (A1) and (A2) hold by the rules of calculus, as $\xi$ is $C^r$. (A3s), (A$\delta$) and (A$\nu$) hold since $t\mapsto t^{1/2d}$ and $\sqrt{\cdot}$ are $C^r$ on $(0,+\infty)$, and $1+f^2>0$, respectively $f^2+g^2>0$ when $\{f\leq0\}\subseteq\{g>0\}$. (A4s) holds by the argument of \cref{lem:Psi}, since being $C^r$ is a local property. (A3w), (A4w) and (A5w) follow from \cref{lem:fischer-regularity}, as $\sigma(f)=f^{2d}$, $\rho(f)\xi(f)=U(f)$ and $\xi(-f)f^{-1}=V(f)$.
\end{proof}

\begin{proof}[Proof of \cref{cor:fischer-strict,cor:fischer-weak}]
Let $\tilde t_0,\ldots,\tilde t_k$ be the terms of \cref{thm:strict}, and define $v_i:M\times S\to R$ by $v_i(x,s):=\tilde t_i\bigl(g_s(x),f_{1,s}(x),\ldots,f_{k,s}(x)\bigr)$. These functions are definable. For each $s\in S$, \cref{lem:fischer-scalar,lem:fischer-algebra} show that \cref{thm:strict} applies to $g_s,f_{1,s},\ldots,f_{k,s}$ in $A$; hence each $v_{i,s}$ is a positive definable $C^r$-function and $g_s=\sigma(v_{0,s})+\sum_i\sigma(v_{i,s})f_{i,s}=v_{0,s}^{2d}+\sum_iv_{i,s}^{2d}f_{i,s}$. The weak case is identical, with the terms of \cref{thm:weak}.
\end{proof}

\section{Concluding remarks}\label{sec:conclusion}

We have isolated an axiomatic core of Fischer's constructive Positivstellens\"{a}tze, proved strict and weak Positivstellens\"{a}tze with universal witnesses in this setting, and obtained instances over definable and continuous functions, including fields that are not real closed and auxiliary functions derived from activation functions. Several questions remain open. The axioms do not cover the polynomial ring $\R[X]$ (\cref{ex:nonexample}), the central setting of the classical Positivstellens\"{a}tze. The lengths of \cref{prop:length} concern the particular constructions considered here, and we do not know whether substantially shorter universal certificates exist. We do not know whether (A5w$^-$) is strictly weaker than (A5w) under the remaining weak axioms (\cref{rem:axiomuse}). It would be of interest to extend \cref{prop:C0} to algebras of $C^r$-functions for $r\geq1$ in the same generality, to find analogues in a noncommutative setting, and to formalise the proofs in a proof assistant.

\section*{Acknowledgments}
We are grateful to three anonymous reviewers of an earlier version, whose comments improved the paper. This work has received funding from the European Union's Horizon Europe research and innovation programme under grant agreement No.~101070568. Tools based on large language models were used in the preparation of this manuscript for brainstorming, checking logical dependencies, revising exposition and identifying possible gaps; the authors take full responsibility for its content.

\bibliographystyle{plainnat}
\bibliography{refs}

\end{document}

%% file: figures/auxiliary_functions.tex
\begin{tikzpicture}[font=\footnotesize]
\pgfplotsset{
  auxpanel/.style={
    scale only axis, width=3.0cm, height=2.2cm,
    axis lines=left, axis line style={-, gray!80},
    tick label style={font=\scriptsize}, title style={font=\footnotesize, yshift=-2pt},
    xlabel={$x$}, xlabel style={font=\scriptsize, at={(ticklabel* cs:1)}, anchor=west},
    execute at end axis={\draw[gray!50, thin] ({rel axis cs:0,0}|-{axis cs:0,0}) -- ({rel axis cs:1,0}|-{axis cs:0,0}); \draw[gray!50, thin] ({axis cs:0,0}|-{rel axis cs:0,0}) -- ({axis cs:0,0}|-{rel axis cs:0,1});},
    xmin=-4, xmax=4, xtick={-4,-2,0,2,4},
    samples=161, no markers, every axis plot/.append style={line width=0.9pt},
    clip=true,
  },
}
\begin{axis}[auxpanel, at={(0,0)}, title={$\sigma(x)=x^2$}, ymin=-1, ymax=17, ytick={0,5,10,15}]
  \addplot[blue!70!black, domain=-4:4] {x^2};
\end{axis}
\begin{axis}[auxpanel, at={(4.3cm,0)}, title={$\rho(x)=\sqrt{|x|}$}, ymin=-0.15, ymax=2.3, ytick={0,1,2}]
  \addplot[blue!70!black, domain=-4:0, samples=200] {sqrt(abs(x))};
  \addplot[blue!70!black, domain=0:4, samples=200] {sqrt(abs(x))};
\end{axis}
\begin{axis}[auxpanel, at={(8.6cm,0)}, title={$\xi(x)=\mathrm{ReLU}(x)^2$}, ymin=-1, ymax=17, ytick={0,5,10,15}]
  \addplot[blue!70!black, domain=-4:4] {max(x,0)^2};
\end{axis}
\begin{axis}[auxpanel, at={(0,-3.55cm)}, title={$\delta(x)=2\sqrt{1+x^2}$}, ymin=-1, ymax=9, ytick={0,2,4,6,8}]
  \addplot[blue!70!black, domain=-4:4] {2*sqrt(1+x^2)};
  \addplot[orange!85!black, dashed, domain=-0.5:4] {2*x};
  \node[orange!85!black, anchor=west, font=\scriptsize] at (axis cs:0.75,0.5) {$2x$};
\end{axis}
\begin{axis}[auxpanel, at={(4.3cm,-3.55cm)}, width=7.3cm, title={$\nu(x,y)=\sqrt{x^2+y^2}+x$ for fixed $y$},
    xmin=-4, xmax=2, xtick={-4,-3,-2,-1,0,1,2}, ymin=-0.3, ymax=5.6, ytick={0,1,2,3,4,5},
    legend cell align=left, legend style={font=\scriptsize, draw=none, fill=white, fill opacity=0.85, text opacity=1, at={(0.02,0.98)}, anchor=north west, row sep=-2pt}]
  \addplot[violet, domain=-4:2] {sqrt(x^2+4)+x};       \addlegendentry{$|y|=2$}
  \addplot[blue!70!black, domain=-4:2] {sqrt(x^2+1)+x}; \addlegendentry{$|y|=1$}
  \addplot[teal!80!black, domain=-4:2] {abs(x)+x};      \addlegendentry{$y=0$}
  \addplot[violet, densely dotted, line width=0.7pt, domain=-4:0] {2};
  \addplot[blue!70!black, densely dotted, line width=0.7pt, domain=-4:0] {1};
\end{axis}
\end{tikzpicture}

%% file: figures/activation_gates.tex
\begin{tikzpicture}[font=\footnotesize,
  declare function={
    asT(\x)=1/(1+0.3275911*abs(\x));
    erfA(\x)=sign(\x)*(1-(0.254829592*asT(\x)-0.284496736*asT(\x)^2+1.421413741*asT(\x)^3-1.453152027*asT(\x)^4+1.061405429*asT(\x)^5)*exp(-(\x)*(\x)));
    gelu(\x)=0.5*\x*(1+erfA(\x/sqrt(2)));
    sigm(\x)=1/(1+exp(-\x));
    softp(\x)=ln(1+exp(\x));
  }]
\pgfplotsset{
  gatepanel/.style={
    scale only axis, width=3.0cm, height=2.4cm,
    axis lines=left, axis line style={-, gray!80},
    tick label style={font=\scriptsize}, title style={font=\footnotesize, yshift=-2pt},
    xlabel={$t$}, xlabel style={font=\scriptsize, at={(ticklabel* cs:1)}, anchor=west},
    execute at end axis={\draw[gray!50, thin] ({rel axis cs:0,0}|-{axis cs:0,0}) -- ({rel axis cs:1,0}|-{axis cs:0,0}); \draw[gray!50, thin] ({axis cs:0,0}|-{rel axis cs:0,0}) -- ({axis cs:0,0}|-{rel axis cs:0,1});},
    samples=161, no markers, clip=true,
  },
  xi/.style={blue!70!black, line width=0.9pt},
  act/.style={orange!85!black, dashed, line width=0.8pt},
}
\begin{axis}[gatepanel, at={(0,0)}, title={$\xi_1(t)=\mathsf{GELU}(t)^2$}, xmin=-2, xmax=2, ymin=-0.6, ymax=4, ytick={0,1,2,3,4},
    legend style={font=\scriptsize, draw=none, fill=white, fill opacity=0.85, text opacity=1, at={(0.02,0.98)}, anchor=north west}]
  \addplot[xi, domain=-2:2] {(x>0)*gelu(x)^2};  \addlegendentry{$\xi_\Phi(t)$}
  \addplot[act, domain=-2:2] {gelu(x)};          \addlegendentry{$\Phi(t)-\Phi(0)$}
  \fill[blue!70!black] (axis cs:0,0) circle (1.3pt);
\end{axis}
\begin{axis}[gatepanel, at={(4.3cm,0)}, title={$\xi_2(t)=(\mathsf{sigmoid}(t)-\frac12)^2$}, xmin=-10, xmax=10, xtick={-10,-5,0,5,10}, ymin=-0.2, ymax=0.6, ytick={0,0.2,0.4,0.6},
    yticklabel style={/pgf/number format/fixed}]
  \addplot[xi, domain=-10:10] {(x>0)*(sigm(x)-0.5)^2};
  \addplot[act, domain=-10:10] {sigm(x)-0.5};
  \fill[blue!70!black] (axis cs:0,0) circle (1.3pt);
\end{axis}
\begin{axis}[gatepanel, at={(8.6cm,0)}, title={$\xi_3(t)=(\mathsf{softplus}(t)-\log2)^2$}, xmin=-2, xmax=2, ymin=-0.65, ymax=2, ytick={0,1,2}]
  \addplot[xi, domain=-2:2] {(x>0)*(softp(x)-ln(2))^2};
  \addplot[act, domain=-2:2] {softp(x)-ln(2)};
  \fill[blue!70!black] (axis cs:0,0) circle (1.3pt);
\end{axis}
\end{tikzpicture}

%% file: main_arxiv_v2.bbl
\begin{thebibliography}{35}
\providecommand{\natexlab}[1]{#1}
\providecommand{\url}[1]{\texttt{#1}}
\expandafter\ifx\csname urlstyle\endcsname\relax
  \providecommand{\doi}[1]{doi: #1}\else
  \providecommand{\doi}{doi: \begingroup \urlstyle{rm}\Url}\fi

\bibitem[Acquistapace et~al.(2002)Acquistapace, Andradas, and Broglia]{AAB}
F.~Acquistapace, C.~Andradas, and F.~Broglia.
\newblock The {P}ositivstellensatz for definable functions on o-minimal
  structures.
\newblock \emph{Illinois J. Math.}, 46\penalty0 (3):\penalty0 685--693, 2002.
\newblock ISSN 0019-2082,1945-6581.
\newblock URL \url{http://projecteuclid.org/euclid.ijm/1258130979}.

\bibitem[Aschenbrenner et~al.(2017)Aschenbrenner, van~den Dries, and van~der
  Hoeven]{ADAMTT}
Matthias Aschenbrenner, Lou van~den Dries, and Joris van~der Hoeven.
\newblock \emph{Asymptotic differential algebra and model theory of
  transseries}, volume 195 of \emph{Annals of Mathematics Studies}.
\newblock Princeton University Press, Princeton, NJ, 2017.
\newblock ISBN 978-0-691-17543-0.
\newblock \doi{10.1515/9781400885411}.
\newblock URL \url{https://doi.org/10.1515/9781400885411}.

\bibitem[Balauca et~al.(2025)Balauca, M{\"u}ller, Mao, Baader, Fischer, and
  Vechev]{balauca2024gaussian}
Stefan Balauca, Mark~Niklas M{\"u}ller, Yuhao Mao, Maximilian Baader, Marc
  Fischer, and Martin Vechev.
\newblock Gaussian loss smoothing enables certified training with tight convex
  relaxations.
\newblock \emph{Transactions on Machine Learning Research}, 2025.

\bibitem[Bingham et~al.(1987)Bingham, Goldie, and Teugels]{BGT_RV}
N.~H. Bingham, C.~M. Goldie, and J.~L. Teugels.
\newblock \emph{Regular variation}, volume~27 of \emph{Encyclopedia of
  Mathematics and its Applications}.
\newblock Cambridge University Press, Cambridge, 1987.
\newblock ISBN 0-521-30787-2.
\newblock \doi{10.1017/CBO9780511721434}.
\newblock URL \url{https://doi.org/10.1017/CBO9780511721434}.

\bibitem[Chen et~al.(2020)Chen, Lasserre, Magron, and
  Pauwels]{NEURIPS2020_dea9ddb2}
Tong Chen, Jean~B Lasserre, Victor Magron, and Edouard Pauwels.
\newblock Semialgebraic optimization for lipschitz constants of relu networks.
\newblock In H.~Larochelle, M.~Ranzato, R.~Hadsell, M.F. Balcan, and H.~Lin,
  editors, \emph{Advances in Neural Information Processing Systems}, volume~33,
  pages 19189--19200. Curran Associates, Inc., 2020.
\newblock URL
  \url{https://proceedings.neurips.cc/paper_files/paper/2020/file/dea9ddb25cbf2352cf4dec30222a02a5-Paper.pdf}.

\bibitem[Chen et~al.(2021)Chen, Lasserre, Magron, and
  Pauwels]{NEURIPS2021_e3b21256}
Tong Chen, Jean~B. Lasserre, Victor Magron, and Edouard Pauwels.
\newblock Semialgebraic representation of monotone deep equilibrium models and
  applications to certification.
\newblock In Marc'Aurelio Ranzato, Alina Beygelzimer, Yann Dauphin, Percy~S.
  Liang, and Jenn~Wortman Vaughan, editors, \emph{Advances in Neural
  Information Processing Systems}, volume~34, pages 27146--27159. Curran
  Associates, Inc., 2021.
\newblock URL
  \url{https://proceedings.neurips.cc/paper_files/paper/2021/file/e3b21256183cf7c2c7a66be163579d37-Paper.pdf}.

\bibitem[De~Rezende et~al.(2025)De~Rezende, Fleming, Janett, Nordstr{\"o}m, and
  Pang]{de2025truly}
Susanna~F De~Rezende, Noah Fleming, Duri~Andrea Janett, Jakob Nordstr{\"o}m,
  and Shuo Pang.
\newblock Truly supercritical trade-offs for resolution, cutting planes,
  monotone circuits, and {W}eisfeiler--{L}eman.
\newblock In \emph{Proceedings of the 57th Annual ACM Symposium on Theory of
  Computing}, pages 1371--1382, 2025.

\bibitem[Dey et~al.(2023)Dey, Dubey, and Molinaro]{dey2023lower}
Santanu~S Dey, Yatharth Dubey, and Marco Molinaro.
\newblock Lower bounds on the size of general branch-and-bound trees.
\newblock \emph{Mathematical Programming}, 198\penalty0 (1):\penalty0 539--559,
  2023.

\bibitem[Dinh and Pham(2021)]{dinh2021nichtnegativstellens}
Si~Tiep Dinh and Tien~Son Pham.
\newblock Nichtnegativstellens\"{a}tze for definable functions in o-minimal
  structures.
\newblock \emph{arXiv preprint arXiv:2105.08278}, 2021.

\bibitem[Dries(1998)]{driesTameTopologyOminimal1998}
Lou van~den Dries.
\newblock \emph{Tame topology and o-minimal structures}, volume 248 of
  \emph{London Mathematical Society Lecture Note Series}.
\newblock Cambridge University Press, Cambridge, 1998.
\newblock ISBN 0-521-59838-9.
\newblock \doi{10.1017/CBO9780511525919}.

\bibitem[Fischer(2011)]{FischerPsatzDiff}
Andreas Fischer.
\newblock Positivstellens\"atze for differentiable functions.
\newblock \emph{Positivity}, 15\penalty0 (2):\penalty0 297--307, 2011.
\newblock ISSN 1385-1292,1572-9281.
\newblock \doi{10.1007/s11117-010-0077-5}.
\newblock URL \url{https://doi.org/10.1007/s11117-010-0077-5}.

\bibitem[Fleming et~al.(2021)Fleming, G{\"o}{\"o}s, Impagliazzo, Pitassi,
  Robere, Tan, and Wigderson]{fleming2021power}
Noah Fleming, Mika G{\"o}{\"o}s, Russell Impagliazzo, Toniann Pitassi, Robert
  Robere, Li-Yang Tan, and Avi Wigderson.
\newblock On the power and limitations of branch and cut.
\newblock In \emph{Proceedings of the 36th Computational Complexity
  Conference}, page~1, 2021.

\bibitem[Gamboa(1987)]{gamboa1987positivstellensatz}
JM~Gamboa.
\newblock A positivstellensatz for rings of continuous functions.
\newblock \emph{Journal of Pure and Applied Algebra}, 45\penalty0 (3):\penalty0
  211--212, 1987.

\bibitem[Huchette et~al.(2026)Huchette, Mu{\~n}oz, Serra, and
  Tsay]{huchette2026deep}
Joey Huchette, Gonzalo Mu{\~n}oz, Thiago Serra, and Calvin Tsay.
\newblock When deep learning meets polyhedral theory: A survey.
\newblock \emph{INFORMS Journal on Computing}, 2026.

\bibitem[Krivine(1964)]{KrivinePsatz}
J.-L. Krivine.
\newblock Anneaux pr\'eordonn\'es.
\newblock \emph{J. Analyse Math.}, 12:\penalty0 307--326, 1964.
\newblock ISSN 0021-7670,1565-8538.
\newblock \doi{10.1007/BF02807438}.
\newblock URL \url{https://doi.org/10.1007/BF02807438}.

\bibitem[Kunc and Kl{\'e}ma(2024)]{kunc2024three}
Vladim{\'\i}r Kunc and Ji{\v{r}}{\'\i} Kl{\'e}ma.
\newblock Three decades of activations: A comprehensive survey of 400
  activation functions for neural networks.
\newblock \emph{arXiv preprint arXiv:2402.09092}, 2024.

\bibitem[Lasserre(2001)]{lasserre2001global}
Jean~B. Lasserre.
\newblock Global optimization with polynomials and the problem of moments.
\newblock \emph{SIAM Journal on Optimization}, 11\penalty0 (3):\penalty0
  796--817, 2001.
\newblock \doi{10.1137/S1052623400366802}.

\bibitem[Lasserre and Putinar(2010)]{lasserre2010positivity}
Jean~B Lasserre and Mihai Putinar.
\newblock Positivity and optimization for semi-algebraic functions.
\newblock \emph{SIAM Journal on Optimization}, 20\penalty0 (6):\penalty0
  3364--3383, 2010.

\bibitem[Mai et~al.(2026)Mai, Magron, Lasserre, and Toh]{mai2026tractable}
Ngoc Hoang~Anh Mai, Victor Magron, Jean-Bernard Lasserre, and Kim-Chuan Toh.
\newblock Tractable hierarchies of convex relaxations for polynomial
  optimization on the nonnegative orthant: Nha mai et al.
\newblock \emph{Computational Optimization and Applications}, pages 1--55,
  2026.

\bibitem[Mao et~al.(2024)Mao, Zhang, and Vechev]{mao2024expressiveness}
Yuhao Mao, Yani Zhang, and Martin Vechev.
\newblock Expressiveness of multi-neuron convex relaxations in neural network
  certification.
\newblock \emph{arXiv preprint arXiv:2410.06816}, 2024.

\bibitem[Marshall and Netzer(2012)]{marshall2012positivstellensatze}
Murray Marshall and Tim Netzer.
\newblock Positivstellens{\"a}tze for real function algebras.
\newblock \emph{Mathematische Zeitschrift}, 270\penalty0 (3):\penalty0
  889--901, 2012.

\bibitem[Miller(2001)]{Miller_Expansions_DLO_IVP}
Chris Miller.
\newblock Expansions of dense linear orders with the intermediate value
  property.
\newblock \emph{J. Symbolic Logic}, 66\penalty0 (4):\penalty0 1783--1790, 2001.
\newblock ISSN 0022-4812,1943-5886.
\newblock \doi{10.2307/2694974}.
\newblock URL \url{https://doi.org/10.2307/2694974}.

\bibitem[M\"{u}ller et~al.(2022)M\"{u}ller, Makarchuk, Singh, P\"{u}schel, and
  Vechev]{1011453498704}
Mark~Niklas M\"{u}ller, Gleb Makarchuk, Gagandeep Singh, Markus P\"{u}schel,
  and Martin Vechev.
\newblock Prima: general and precise neural network certification via scalable
  convex hull approximations.
\newblock \emph{Proc. ACM Program. Lang.}, 6\penalty0 (POPL), January 2022.
\newblock \doi{10.1145/3498704}.
\newblock URL \url{https://doi.org/10.1145/3498704}.

\bibitem[Palma et~al.(2024)Palma, Bunel, Dvijotham, Kumar, Stanforth, and
  Lomuscio]{palma2024expressive}
Alessandro~De Palma, Rudy~R Bunel, Krishnamurthy~Dj Dvijotham, M.~Pawan Kumar,
  Robert Stanforth, and Alessio Lomuscio.
\newblock Expressive losses for verified robustness via convex combinations.
\newblock In \emph{The Twelfth International Conference on Learning
  Representations}, 2024.
\newblock URL \url{https://openreview.net/forum?id=mzyZ4wzKlM}.

\bibitem[Parrilo(2003)]{parrilo2003semidefinite}
Pablo~A. Parrilo.
\newblock Semidefinite programming relaxations for semialgebraic problems.
\newblock \emph{Mathematical Programming}, 96\penalty0 (2):\penalty0 293--320,
  2003.
\newblock \doi{10.1007/s10107-003-0387-5}.

\bibitem[Plump(1999)]{plump1999term}
Detlef Plump.
\newblock Term graph rewriting.
\newblock \emph{Handbook Of Graph Grammars And Computing By Graph
  Transformation: Volume 2: Applications, Languages and Tools}, pages 3--61,
  1999.

\bibitem[Putinar(1993)]{PutinarPsatz}
Mihai Putinar.
\newblock Positive polynomials on compact semi-algebraic sets.
\newblock \emph{Indiana Univ. Math. J.}, 42\penalty0 (3):\penalty0 969--984,
  1993.
\newblock ISSN 0022-2518,1943-5258.
\newblock \doi{10.1512/iumj.1993.42.42045}.
\newblock URL \url{https://doi.org/10.1512/iumj.1993.42.42045}.

\bibitem[Scheiderer(2024)]{scheiderer2024course}
Claus Scheiderer.
\newblock \emph{A course in real algebraic geometry---positivity and sums of
  squares}, volume 303 of \emph{Graduate Texts in Mathematics}.
\newblock Springer, Cham, 2024.
\newblock ISBN 978-3-031-69212-3; 978-3-031-69213-0.
\newblock \doi{10.1007/978-3-031-69213-0}.
\newblock URL \url{https://doi.org/10.1007/978-3-031-69213-0}.

\bibitem[Schm\"udgen(1991)]{SchmudgenPsatz}
Konrad Schm\"udgen.
\newblock The {$K$}-moment problem for compact semi-algebraic sets.
\newblock \emph{Math. Ann.}, 289\penalty0 (2):\penalty0 203--206, 1991.
\newblock ISSN 0025-5831,1432-1807.
\newblock \doi{10.1007/BF01446568}.
\newblock URL \url{https://doi.org/10.1007/BF01446568}.

\bibitem[Singh et~al.(2018)Singh, Gehr, Mirman, P\"{u}schel, and
  Vechev]{NEURIPS2018_f2f44698}
Gagandeep Singh, Timon Gehr, Matthew Mirman, Markus P\"{u}schel, and Martin
  Vechev.
\newblock Fast and effective robustness certification.
\newblock In S.~Bengio, H.~Wallach, H.~Larochelle, K.~Grauman, N.~Cesa-Bianchi,
  and R.~Garnett, editors, \emph{Advances in Neural Information Processing
  Systems}, volume~31. Curran Associates, Inc., 2018.
\newblock URL
  \url{https://proceedings.neurips.cc/paper_files/paper/2018/file/f2f446980d8e971ef3da97af089481c3-Paper.pdf}.

\bibitem[Stengle(1974)]{StenglePsatz}
Gilbert Stengle.
\newblock A nullstellensatz and a positivstellensatz in semialgebraic geometry.
\newblock \emph{Math. Ann.}, 207:\penalty0 87--97, 1974.
\newblock ISSN 0025-5831,1432-1807.
\newblock \doi{10.1007/BF01362149}.
\newblock URL \url{https://doi.org/10.1007/BF01362149}.

\bibitem[Wang et~al.(2021)Wang, Magron, and Lasserre]{wang2021chordal}
Jie Wang, Victor Magron, and Jean-Bernard Lasserre.
\newblock Chordal-tssos: a moment-sos hierarchy that exploits term sparsity
  with chordal extension.
\newblock \emph{SIAM Journal on optimization}, 31\penalty0 (1):\penalty0
  114--141, 2021.

\bibitem[Wang et~al.(2024)Wang, Hu, Havens, Araujo, Zheng, Chen, and
  Jha]{wang2024scalability}
Zi~Wang, Bin Hu, Aaron~J Havens, Alexandre Araujo, Yang Zheng, Yudong Chen, and
  Somesh Jha.
\newblock On the scalability and memory efficiency of semidefinite programs for
  lipschitz constant estimation of neural networks.
\newblock In \emph{The Twelfth International Conference on Learning
  Representations}, 2024.

\bibitem[Xue et~al.(2022)Xue, Lindemann, Robey, Hassani, Pappas, and
  Alur]{xue2022chordal}
Anton Xue, Lars Lindemann, Alexander Robey, Hamed Hassani, George~J Pappas, and
  Rajeev Alur.
\newblock Chordal sparsity for {L}ipschitz constant estimation of deep neural
  networks.
\newblock In \emph{2022 IEEE 61st Conference on Decision and Control (CDC)},
  pages 3389--3396. IEEE, 2022.

\bibitem[Zhang et~al.(2022)Zhang, Wang, Xu, Li, Li, Jana, Hsieh, and
  Kolter]{NEURIPS2022_0b06c867}
Huan Zhang, Shiqi Wang, Kaidi Xu, Linyi Li, Bo~Li, Suman Jana, Cho-Jui Hsieh,
  and J.~Zico Kolter.
\newblock General cutting planes for bound-propagation-based neural network
  verification.
\newblock In S.~Koyejo, S.~Mohamed, A.~Agarwal, D.~Belgrave, K.~Cho, and A.~Oh,
  editors, \emph{Advances in Neural Information Processing Systems}, volume~35,
  pages 1656--1670. Curran Associates, Inc., 2022.
\newblock URL
  \url{https://proceedings.neurips.cc/paper_files/paper/2022/file/0b06c8673ebb453e5e468f7743d8f54e-Paper-Conference.pdf}.

\end{thebibliography}
